\documentclass[12pt]{article}
\usepackage[english]{babel}

\usepackage{amsmath}
\usepackage{amssymb}
\usepackage{amsthm}
\usepackage{amsfonts}
\usepackage{graphicx}
\usepackage{comment}
\usepackage{enumitem}
\usepackage{color}
\usepackage[colorlinks=true, allcolors=blue]{hyperref}

\usepackage{tikz-cd}

\usepackage{oldgerm}
\usepackage[text={15.5cm,22.5cm},centering]{geometry}

\newcommand{\id}{\operatorname{id}}
\usepackage{subcaption}
\usepackage{caption} % Include the caption package

\DeclareMathAlphabet\mathbfcal{OMS}{cmsy}{b}{n}

\newtheorem{thm}{Theorem}[section]
\newtheorem{Thm}[thm]{Theorem}

\newtheorem{Lemma}[thm]{Lemma}

\newtheorem{Prop}[thm]{Proposition}

\newtheorem*{claim*}{Claim}

\newtheorem{Cor}[thm]{Corollary}
\newtheorem{Open}[thm]{Open Problem}

\theoremstyle{definition}
\newtheorem{defn}[thm]{Definition}
\newtheorem{ex}[thm]{Example}
\newtheorem{rmrk}[thm]{Remark}
\newtheorem{fact}[thm]{Fact}
\newenvironment{Example}{\begin{ex}}{\hfill $\dashv$\end{ex}}
\newenvironment{Def}{\begin{defn}}{\hfill $\dashv$\end{defn}}
\newenvironment{Remark}{\begin{rmrk}}{\hfill $\dashv$\end{rmrk}}

\DeclareMathOperator{\dom}{dom}

\DeclareMathOperator{\pr}{pr}
\newenvironment{enumerate-(a)}{\begin{enumerate}[label={\upshape (\alph*)}, leftmargin=2pc]}{\end{enumerate}}
\newenvironment{enumerate-(a)-r}{\begin{enumerate}[label={\upshape (\alph*)}, leftmargin=2pc,resume]}{\end{enumerate}}
\newenvironment{enumerate-(a)-5}{\begin{enumerate}[label={\upshape (\alph*)}, leftmargin=2pc,start=5]}{\end{enumerate}}
\newenvironment{enumerate-(A)}{\begin{enumerate}[label={\upshape (\Alph*)}, leftmargin=2pc]}{\end{enumerate}}
\newenvironment{enumerate-(A)-r}{\begin{enumerate}[label={\upshape (\Alph*)}, leftmargin=2pc,resume]}{\end{enumerate}}
\newenvironment{enumerate-(i)}{\begin{enumerate}[label={\upshape (\roman*)}, leftmargin=2pc]}{\end{enumerate}}
\newenvironment{enumerate-(i)-r}{\begin{enumerate}[label={\upshape (\roman*)}, leftmargin=2pc,resume]}{\end{enumerate}}
\newenvironment{enumerate-(I)}{\begin{enumerate}[label={\upshape (\Roman*)}, leftmargin=2pc]}{\end{enumerate}}
\newenvironment{enumerate-(I)-r}{\begin{enumerate}[label={\upshape (\Roman*)}, leftmargin=2pc,resume]}{\end{enumerate}}
\newenvironment{enumerate-(1)}{\begin{enumerate}[label={\upshape (\arabic*)}, leftmargin=2pc]}{\end{enumerate}}
\newenvironment{enumerate-(1)-r}{\begin{enumerate}[label={\upshape (\arabic*)}, leftmargin=2pc,resume]}{\end{enumerate}}

\newenvironment{enumerate-(star)}{\begin{enumerate}[label={\upshape{(\( \star_{ \arabic*} \))}}, leftmargin=2pc]}{\end{enumerate}}
\renewcommand{\d}{\delta}

\newcommand{\Q}{\mathbb{Q}} % rationals
\newcommand{\R}{\mathbb{R}} %reals
\newcommand{\N}{\mathbb{N}} %natural numbers
\newcommand{\UU}{\mathcal{U}}

\newcommand{\YY}{\mathcal{Y}}

\newcommand{\II}{\mathcal{I}}

\newcommand{\XX}{\mathcal{X}}

\usepackage{stackengine,scalerel}
\stackMath

\mathchardef\mhyphen="2D
\newcommand{\rest}{\!\restriction\!}
\newcommand{\homeo}{\approx}

\newcommand{\es}{\varnothing}
\renewcommand{\ge}{\geqslant}
\renewcommand{\le}{\leqslant}

\newcommand{\e}{\varepsilon}

\renewcommand{\Cup}{\bigcup}

\usepackage{mwe}
\usepackage{authblk}
\usepackage{afterpage}
\title{Equivalent Environments and Covering Spaces\\ for Robots}
\author{Vadim K. Weinstein and Steven M. LaValle}
\affil{Faculty of Information Technology and Electrical Engineering\\ University of Oulu, Finland}
\begin{document}

\maketitle
\thispagestyle{empty}

\begin{abstract}
This paper formally defines a robot system, including its sensing and actuation components, 
as a general, topological dynamical system. 
The focus is on determining general conditions 
under which various environments in which the robot can be placed are indistinguishable.
A key result is that, under very general conditions, covering maps witness such indistinguishability. 
This formalizes the intuition behind the well studied loop 
closure problem in robotics. An important special case is where the sensor mapping
reports an invariant of the local topological (metric) structure of an environment because
such structure is preserved by (metric) covering maps. Whereas coverings provide a sufficient condition for the equivalence of environments,
we also give a necessary condition using bisimulation. 
The overall framework is applied to unify previously 
identified phenomena in robotics and related fields, in which moving agents with sensors 
must make inferences about their environments based on limited data.  Many open problems 
are identified.
\end{abstract}

\newpage
\tableofcontents

\newpage

\section{Introduction}\label{sec:intro}

\begin{figure}
  \centering
  \includegraphics[width=0.7\textwidth]{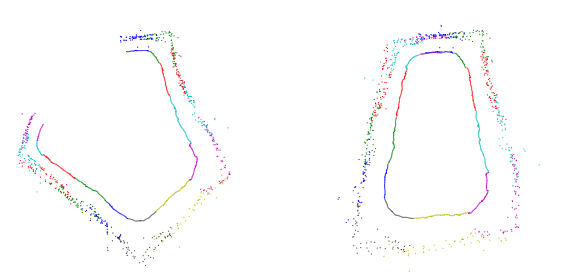}
  \caption{Map reconstruction before loop closure (left) and after
    loop closure (right). Note how a new metric has been defined on
    the space to support the new information that the distance of two
    previously distinct points equals now to zero. Figure reproduced
    from~\cite{williams08}.}
  \label{fig:Loop}
\end{figure}

When a mobile robot explores a multiply connected environment using
sensors, it frequently encounters the well-known problem of \emph{loop
closure}, in which it must detect that it has returned to a previously
visited location.  Robots combine information from sensor observations at multiple times, leading to the {\em filtering} problem (also known as {\em sensor fusion}) of appropriately aggregating or summarizing its data.  Shown in Figure \ref{fig:Loop}, the robot becomes
mistaken (from our perspective) about its position over time as it tries to build a geometric
representation of its environment.  The problem is part of 
\emph{SLAM (simultaneous localization and mapping})~\cite{DisNewClaDurCso01,HahFoxBurThr03}, 
a critical operation in many robots and autonomous systems.  Similarly, a person donning a
VR headset can be tricked into walking in circles when they believe
they are heading straight, using a method known as \emph{redirected walking}~\cite{Suma2013}.  Likewise, according to a popular theory~\cite{Sotthibandhu1979}, a moth inadvertently travels in circles around a fire because it cannot distinguish the fire and a celestial light source%
\protect{\footnote{This theory was recently disputed in~\cite{Fabian2024}.}. 
This phenomenon occurs, and can be studied, in the framework of
minimalist robotics such as Braitenberg robots~\cite{Bra84}, wall-following robots
\cite{Katsev2011,Taylor2013}, and robots with topological sensing and filtering, 
such as the gap-navigating robots of~\cite{TovMurLav07}. It is also closely
related to some version of the graph exploration problem~\cite{Kuipers1988ARQ,Thrun1998}.

All of these examples (and many more) share an underlying principle: 
The covering map $f\colon \R\to S^1$, $t\mapsto e^{it}$ commutes with the 
agent's sensorimotor behavior, whether the agent be a robot, human, or insect.  
More generally, topological ambiguities arise from limited sensing and actuation capabilities, and we show in this paper that they are naturally understood in terms of covering spaces.
Figure~\ref{fig:Indistinguishable} shows examples of spaces and their coverings
which, considered as robot's environments, will be shown to be strongly indistinguishable
(Section~\ref{ssec:Strongly}). As a converse we obtain an invariant of the purely
topological equivalence relation between spaces which holds if and only if 
they have a common covering space,
see Figure~\ref{fig:AppToT}.

We approach this through the abstract
theory of information spaces which was originally
proposed and developed in \cite{LaValle2006,Tovar2005} and extended recently 
in~\cite{Sakcak2023,WeiSakLav22}. 
In it, the robot (or more broadly, an agent), has access exclusively to the sequence of sensorimotor
interactions with the environment given by a sequence
$\eta=(u_0,y_0,u_1,y_1,\ldots,u_n,y_{n-1})$ in which $u_k$ is the
motor input at stage $k$ and $y_k$ is the sensory data at stage~$k$.
These sequences are called \emph{history information states}.  
The robot may be thought of exploring a tree of all possible
history information states ordered by end-extension.  This tree is
called the \emph{history information space}. 

\begin{figure}
  \centering
  \includegraphics[width=0.8\textwidth]{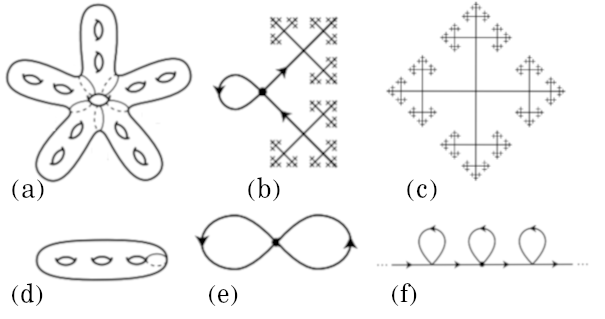}
  \caption{The spaces on the bottom are strongly indistinguishable
    from the ones above them. The 2-manifold (a) is a covering space
    of (d), the 1-complex (b) is a covering space of (e) and (f), whereas
    (c) is the universal covering space of (b), (e), and (f). For a robot, which
    senses the local homeomorphism type, (b), (c), (e), and (f)
    are mutually (not strongly) indistinguishable. 
    Figures are reproduced from~\cite{AT}.}
  \label{fig:Indistinguishable}
\end{figure}

\begin{figure}
  \centering
  \includegraphics[width=0.5\textwidth]{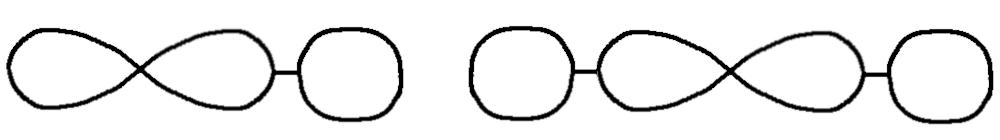}
  \caption{A robot, which senses the local homeomorphism type, can distinguish between these spaces, which
    implies that they cannot have a common covering space (see Example~\ref{ex:ApplicationToTop}).
    This constitutes an application of the present theory of mathematical robotics to topology.}
  \label{fig:AppToT}
\end{figure}

We begin by reviewing prior work on loop closure, SLAM, graph
exploration, and minimal filtering in Section~\ref{sec:Loops}.
In Section~\ref{sec:GeneralTheory} we
develop the general theory of the spaces of control signals, trajectories, and
what we call \emph{path actions}. 
We review some prior work where such
spaces have been defined
(Sections~\ref{ssec:ControlSignals}-\ref{ssec:SpaceOfTrajectories})
and strengthen some of the results of \cite{Yershov2010} in order
to motivate the more general topological definitions
in Section~\ref{sec:Topological}. Our main result in this section 
is Corollary~\ref{cor:ContinuousPathAction} which says that
the map assigning robot's trajectories to control signals
is continuous under the assumption that the mere position
of the robot is a continuous function of the control signals.
We prove this under very general topological assumptions
on the set of control signals (Definition~\ref{def:TopoU}).

Then, in Section~\ref{sec:HISS} we
apply the framework developed in \ref{sec:GeneralTheory}-\ref{sec:Topological} 
to define the notion of \emph{sensorimotor structure}
as well as the general notion of \emph{environment}. The sensorimotor structure
consists of a sensor mapping $h$ and a path action $p$ which define the coupling
between the ambient space and the robot's sensors and actuators
(Section~\ref{ssec:Life-paths}). Then, we define the continuous version of history
information state space -\ref{ssec:ContHIS}) generalizing the discrete
time version of~\cite{LaValle2006}.

We then use these tools to define a class of 
indistinguishability
equivalence relations on environments in
Section~\ref{sec:Indistinguishability} and prove a number of basic
results about them. These are interesting 
from the perspective of both robotics and topology. The most basic
one, which we also analyze the most, says that two environments are indistinguishable, if
any control signal
(no matter how long) will yield identical sensory readings in both of them
(Definitions~\ref{def:I-equiv} and~\ref{def:LifePathEq}).

In Section~\ref{sec:HomoCover} we introduce the idea of covering
spaces in this context and prove our next main results (Theorem~\ref{thm:Lift}, 
Corollary~\ref{cor:Lift}, and Theorem~\ref{thm:CoveringEquiv}) which 
formalize the idea that a covering map enables lifting the sensorimotor structure
to make the covered and covering spaces indistinguishable, or that environments
which have a common covering (or are both coverings of an) environment, then they are
also indistinguishable.

In the end of the paper we prove our final main result, Theorem~\ref{thm:Bisimulation},
giving a complete characterization of the $\equiv$-equivalence in
terms of bisimulations. Section~\ref{sec:Revisiting}
then revisits the examples of Section~\ref{sec:Loops}. Finally we
prove a general result about robots whose sensors report an isometry-invariant
of their local neighbourhood, Theorem~\ref{thm:IsometryInvariant}.

\section{Closing the Loop via Sensing and Filtering}
\label{sec:Loops}

This section reviews the well-known loop closure problem
in robotics, both in the context of engineering (Section~\ref{ssec:LCinSLAM}),
and as analyzed in the theoretical minimalist setup (Section~\ref{ssec:MinimalFiltering}).
Finally, we touch upon the related topic of graph exploration in 
Section~\ref{ssec:GraphExpl}.

\subsection{Loop closure in SLAM}
\label{ssec:LCinSLAM}

In simultaneous localization and mapping (SLAM) problems, robots are
moving in their environments while also scanning them. They may also
keep track of various non-visual parameters such as data from their
accelerometer, GPS-coordinates, wheel encoders, sonars, lasers and so
on~\cite{Tsintotas2022}. Once the robot circles back to an area where
it has been before, it is incumbent on the robot to be capable of
``meeting the ends'', that is, identifying the place as one which has
been already visited and applying this knowledge to the map
reconstruction problem. The \emph{visual loop closure} is the subcategory of
loop closure problems in which this recognition has to be made based
on visual sensors alone (possibly including 3D-sensors such as
lasers). Even when motion detection data is available, accumulated
errors may make loop closure detection difficult based on path
integration methods. This error makes the loop closure problem
in essence topological, as opposed to metric. In fact, once two
places in the internally reconstructed \emph{metric} maps have been
identified as the same, usually a new metric has to be defined on the
map to make it correct, see Figure~\ref{fig:Loop}

In many applications the visual sensors are powerful monocular or
binocular multi-megapixel cameras and laser
measurements~\cite{Chen2006,Geiger2013}. Many computational methods
have been developed to perform this highly complex task.  Methods
which work directly with data include statistical such as Bayesian
inference \cite{Chen2006}, hybrid statistical-geometric such as direct
sparse odometry~\cite{Engel18,Gao18}, graph based methods such as
local registration and global correlation \cite{Gutman99} and other
sophisticated methods such as ORB-SLAM~\cite{MurArtal15} and
graph-matching with deep learning~\cite{Duan2022}. These include scan
matching, consistent pose estimation, map correlation, clustering,
and dimension analysis.  There is also a class of methods which are based
on image recognition where the visual images are labeled using neural
networks and then the algorithms operate with those labels.

\subsection{Minimal filtering and minimal SLAM}
\label{ssec:MinimalFiltering}

At the other extreme, minimalist models which analyze the logical and
geometric underpinnings of SLAM have been proposed
\cite{Katsev2011,Taylor2013,TovLavMut03,TovMurLav07}.  In this line of
work the analysis is performed in mathematically well-defined
environments and the models assume only simple feature detectors.
The robots are often capable of building very minimalist, topological
representations of the environment
so that the required tasks are accomplished reliably and provably. 
Loop closure features in this theory too.
It can be achieved by the robot
dropping recognizable pebbles in the environment at selected
locations~\cite{TovLavMut03,TovMurLav07}. The framework of minimalist
robotics has the theoretical advantage that it enables one to reason
about the limiting cases pertaining in particular to whether loop
closure is theoretically possible or not.

Consider gap-navigation trees introduced in~\cite{TovMurLav07}. Here,
the robot is moving in a planar environment $O\subset \R^2$ and is
equipped with a sensor which only reports the directions in which the
distance-to-the-boundary function has discontinuities. These correspond to
corners and turns in the boundary of the environment. This sensor can
be thought of as detecting a topological invariant of the star-convex
environment of the robot's current location
(Figure~\ref{fig:StarConvexInvariant}). We will make this statement precise
in Example~\ref{ex:GapNavTrees}. Using this data the robot is
able to internally build a tree-like model which encodes the convex
region structure sufficiently well for the robot to be able to
optimally solve navigation tasks in a simply connected environment.
This setup is subject to the loop-closure problem. If the environment
is not simply connected, the robot may start going around a circular
obstacle, but updating the model as if it is seeing new regions ad infinitum. 
It was shown in  \cite{TovMurLav07} that their model will fail
in non-simply connected environments. We show in Example~\ref{ex:GapNavTrees}
the same thing within our new framework.
By equipping the robot with loop closure detection (using pebbles), the authors of \cite{TovMurLav07} show that then the
navigation problem can be solved in non-simply connected environments too. 

\begin{figure}
  \centering
  \includegraphics[width=0.4\textwidth]{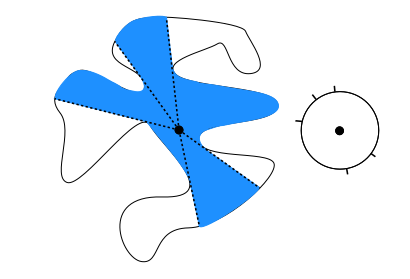}
  \caption{The visible star-convex region (left) and the outcome of
    sensor filtering (right). The filter is a topological invariant of
    the star-region as we will see in Section~\ref{sec:Revisiting}.
    Image reproduced from~\cite{TovMurLav07}}
  \label{fig:StarConvexInvariant}
\end{figure}

\begin{figure}
  \centering
  \includegraphics[width=0.4\textwidth]{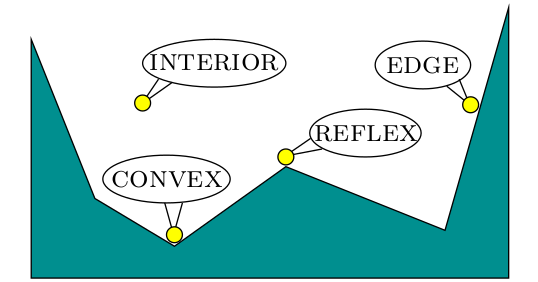}
  \caption{An example of a sensor mapping which is a
    metric invariant of the local neighbourhood of the robot.
    See Section~\ref{sec:Revisiting} for more details.
    Image reproduced from~\cite{Katsev2011}}
  \label{fig:ConvexReflex}
\end{figure}

In~\cite{Katsev2011} the authors analyze a simple robot with local
sensors that moves in an unknown polygonal environment. The robot is
capable of sensing local geometric structure of the environment: it
can detect whether or not it is on the boundary and if it is, whether
or not it is on a vertex and if it is, whether
it is a convex or a reflex (concave) vertex (Figure~\ref{fig:ConvexReflex}). The robot can
also leave pebbles in the environment which it can later detect. Using
this machinery the robot is shown to be capable of various tasks. Our
interest in this example is that here also, the sensor is a geometric
(in this case metric, not topological) invariant of the local
neighbourhood of the robot, see Section~\ref{sec:Revisiting}.

\begin{figure}
  \centering
  \includegraphics[width=0.4\textwidth]{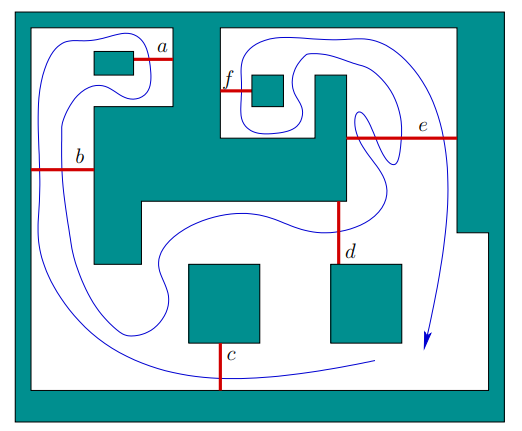}
  \caption{Navigation using only the data from the beams.  Image
    reproduced from~\cite{Tovar2009}}
  \label{fig:Beams}
\end{figure}

In \cite{Tovar2009} the authors ``analyze a problem in which an unpredictable moving body travels among obstacles and binary detection
beams. The task is to determine the possible body path based only on
the binary sensor data.  This is a basic filtering problem encountered
in many settings, which may arise from physical sensor beams or
virtual beams that are derived from other sensing modalities.'' (quote
from~\cite{Tovar2009}, see also Figure~\ref{fig:Beams}) The authors
show among other things that if the region is partitioned by the beams
into simply connected regions, then the body can know the homotopy
type of the traversed trajectory.

\subsection{Graph exploration}
\label{ssec:GraphExpl}

Graph exploration is a family of problems in the intersection of graph
theory and theoretical robotics. A subcategory of them
ask the robot to generate a map of the
environment, either a complete one or a sufficient one for its
tasks~\cite{Kuipers1988ARQ,Thrun1998}.

This set of problems must include some form of loop
closure, for otherwise a robot cannot distinguish between the 
graphs (b), (c), (e), and (f) shown on Figure~\ref{fig:Indistinguishable}.
The theory outlined in this paper 
applies to this context in the limited special case when the robot is not
allowed to make marks in the environment.

% \subsection{Loop closure in VR}

% In \emph{redirected walking} VR researchers have been studying how to
% design the perceptual space (the sensing function in our framework) in
% such a way that the subject traverses a different path than they
% think~\cite{Suma2013}. The practical reason for doing so is to enable
% large virtual environments to be realizable in smaller physical spaces. 
% One particular application is to have the
% subject walk in circles while in the VR-world they are walking along a
% straight line. This is achieved by exploiting visual illusions and
% slight turns in the virtual environment which are not consciously
% detected by the subject but are reflected in their walking behaviour.
% \begin{quote}
%   When the user rotates his head, the change in head orientation is
%   measured by the system, and a scale factor is applied to the
%   rotation in the virtual world. The net result is a gradual rotation
%   of the entire virtual world around the user's head position, which
%   in turn alters their walk direction in the real world. ...  [T]he
%   virtual world can also be rotated as the user walks in a straight
%   line in the virtual world, a manipulation known as curvature
%   gains.~\cite{Suma2013}
% \end{quote}

\section{Basic Models and Theory Development}
\label{sec:GeneralTheory}

In this section we introduce the basic setup. In Section~\ref{ssec:ControlSignals}
we introduce the space $\UU_M$ of measurable control signals, and
in Section~\ref{ssec:SpaceOfTrajectories} the space $\XX$ of continuous
trajectories that robot traverses in~$X$. We analyze, building on 
\cite{Yershov2010} the case when $X$ is a differential manifold
and define the functions $r$ and $\bar r$ which connect the control signal
to the resulting configuration and the trajectory of the robot respectively. 
Using the results of \ref{ssec:ControlSignals}
and \ref{ssec:SpaceOfTrajectories}, we then formulate appropriate generalizations 
to a topological framework in Section~\ref{sec:Topological}.
There, we define $\UU$ which is a generalization of $\UU_M$,
and the path action $p$ (and the induced $\bar p$)
which is a generalization of $r$ (respectively of~$\bar r$).
The space $X$ is merely generalized from being a differential manifold
to being a Polish space without any changes in the definition of~$\XX$.

\subsection{Space of control signals}
\label{ssec:ControlSignals}

This section will follow~\cite{Yershov2010} in defining the space $\UU$ of
control signals. Later, in Section~\ref{sec:Topological}, we will have
a bit more general definition. The robot has an arbitrary, nonempty set of 
inputs~$U$ which in this section is a topological space.  A \emph{control
  signal} is a measurable function $\bar u\colon [0,T)\to U$.  We use
half-open intervals as the domain (deviating from \cite{Yershov2010})
so that we can naturally concatenate them. Over its lifetime, the robot
computer or controller generates a signal $\bar x\colon \R_{\ge 0}\to U$, which
can be obtained as a concatenation of bounded-time signals. This will be made formal.

\begin{Def}\label{def:PathSpaceBasics}
  Let $\UU_M$ be the set of all measurable functions
  $\bar u\colon [0,T)\to U$ for $T\in \R_{\ge 0}$.  For $T=0$ this is
  the empty function $\bar u=\es$.  Denote by $T=|\bar u|$ the
  supremum of the domain of~$\bar u$. For any two elements
  $\bar u_1,\bar u_2\in \UU_M$, let $\bar u_1\oplus\bar u_2$ be the
  control signal $\bar u$ with $|\bar u|=|\bar u_1|+|\bar u_2|$
  defined by
  $$\bar u(t)=
  \begin{cases}
    \bar u_1(t)&\text{ if }t<|\bar u_1|\\
    \bar u_2(t-|\bar u_1|)&\text{ otherwise}.
  \end{cases}
  $$
  Given $\bar u\in \UU_M$, denote by $\bar u\rest [0,t)$ the restriction
  of $\bar u$ to $[0,t)$. This can be shortened to $\bar u\rest t$
  or to $\bar u_{<t}$. If $t\ge |\bar u|$, then $\bar u\rest t=\bar u$.
  Let $\bar u_1\lhd \bar u_2$ denote that $|\bar u_1|<|\bar u_2|$
  and $\bar u_2\rest|\bar u_1|=\bar u_1$. Given $\bar u\in \UU_M$ and
  $t\in \R_{\ge 0}$ we define $\bar u_{< t}$ and $\bar u_{\ge t}$ to
  be the unique two elements of $\UU_M$ with the property that
  $|\bar u_{<t}|=t$ and $\bar u = \bar u_{<t}\oplus \bar u_{\ge t}$.
\end{Def}

\begin{Prop}\label{prop:UU_Mclosedunder}
    $\UU_M$ is 
    \begin{enumerate-(1)-r} 
    \item closed under $\oplus$, meaning
    that for all $\bar u_0,\bar u_1\in\UU_M$ we have
    $\bar u_0\oplus\bar u_1\in \UU_M$,
    \item \emph{closed under segmentation}, meaning
    that for all $\bar u$ and all $t$, also $\bar u_{<t}$
    and $\bar u_{\ge t}$ are in $\UU_M$, and 
    \item \emph{extensive},
    meaning that for all $\bar u\in\UU_M$, there is $\bar u_1\in\UU_M$
    such that $\bar u\lhd \bar u_1$ and $|\bar u_1|\ge |\bar u|+1$.
    \end{enumerate-(1)-r}
\end{Prop}
\begin{proof}
    These all follow easily from the fact that $\UU_M$ is the collection
    of all measurable functions
\end{proof}

\begin{Remark}\label{rem:Basics}
  \begin{enumerate}
  \item As is standard in functional analysis, we actually consider equivalence classes of measurable functions that differ on a set of
    measure zero, without altering notation.
  \item The ordering $\lhd$ is a tree order on~$\UU_M$, meaning 
    $\lhd$-incompatible control signals have no common $\lhd$-extensions.
  \item We always have $\bar u_1\lhd\bar u_1\oplus\bar u_2$.
  \item Also, $(\bar u_0\oplus \bar u_1)_{<|u_0|}=u_0$ and
    $(\bar u_0\oplus \bar u_1)_{\ge |u_0|}=u_1$.
  \item $\bar u_{<t}$ is the same as the restriction $\bar u\rest [0,t)$.
  \item $\bar u_{\ge t}$ is the final segment of $\bar u$ which can be
    explicitly defined as a function with domain $[0,t_0)$ where
    $t_0=|\bar u|-t$ and for all $0\le t_1<t_0$ we have
    $\bar u_{\ge t}(t_1)= \bar u(t+t_1).$
  \end{enumerate}
\end{Remark}

Following \cite{Yershov2010} we define the $L_1$-type metric
on~$\UU_M$. 

\begin{Def}[Metric on $\UU_M$]\label{def:MetricOnU}
  Assume that $d_U$ is a metric on~$U$. Then,
  given measurable $\bar u_1\colon [0,T_1)\to U$
  and $\bar u_2\colon [0,T_2)\to U$, define
  $$\varrho_{\UU_M}(\bar u_1,\bar u_2)=\int_0^{T}d_U(\bar u_1(t),\bar u_2(t))dt+|T_1-T_2|,$$
  where $T=\min\{T_1,T_2\}$.  
\end{Def}

% For the sake of completeness we sketch the proof of the following
% which was omitted in~:

% \begin{Lemma}(\cite[Lemma~2]{Yershov2010}) If $d_U$ is a bounded
%   metric on $U$, then $\varrho_{\UU_M}$ is a metric on~$\UU_M$.
% \end{Lemma}
% \begin{proof}
%   Since $d_U$ has a bound, say $M\in \R$, we have
%   $$\varrho_{\UU_M}(\bar u_1,\bar u_2)\le \int_0^{T}Mdt+|T_1-T_2|=M\min\{T_1,T_2\}+|T_1-T_2|<\infty,$$
%   so the integral converges. If $\bar u_1=\bar u_2$, then clearly both
%   terms in the integral are zero and the result is zero. If, on the
%   other hand $\bar u_1\ne \bar u_2$, then these functions either
%   differ in a set $A\subset [0,T)$ of positive measure, or else the
%   term $|T_2-T_1|$ is positive. In either case
%   $\varrho_{\UU_M}(\bar u_1,\bar u_2)>0$.  Symmetry follows from the
%   symmetries of the metrics on $U$ and~$\R$.  The triangle inequality
%   also follows from the triangle inequalities for metrics on $U$ and
%   $\R$ together with the linearity of the integral.
% \end{proof}

A metric is called \emph{Polish} if it makes the space complete and
separable.

\begin{Prop}\label{prop:UPolish}
  If $d_U$ is bounded, then $\varrho_{\UU_M}$ is a Polish metric on~$\UU_M$.
\end{Prop}
\begin{proof}
  That it is a metric was already observed in~\cite{Yershov2010}.
  A dense countable set was also constructed in~\cite{Yershov2010}.
  Suppose $(\bar u_{n})$ is a Cauchy sequence. Because of the second
  term in the definition of the metric, $|\bar u_n|$ must converge to
  some~$T$. Let $(T_n)$ be an increasing sequence converging to $T$ with
  the property that for all $m\ge n$ we have $|\bar u_m|>T_n$. Then
  for each $n$, the sequence $(\bar u_{m}\rest_{<T_n})_{m\ge n}$ is a
  Cauchy sequence in the classical $L_1$-metric on the space
  of all continuous functions from $[0,T_n)$ into $U$ which
  is known to be complete.  Thus, let $\bar u^*_n$ be the limit of
  that sequence. Then, observe that $\bar u^*_n\lhd \bar u^*_{n+1}$
  for all $n$, and $\Cup_{n}\bar u^*_n$ is the limit of~$(\bar u_n)$ in~$\UU_M$.
\end{proof}

The following acts as a motivation for the generalization in Definition~\ref{def:TopoU}.

\begin{Prop}\label{lemma:ProjectionsContinuous}
  For $\bar u\in\UU_M$ and  $t\in\R_{\ge 0}$,
  the following functions are continuous:
  \begin{enumerate}
  \item[(1)] $\tau_1\colon \bar u\mapsto |\bar u|$,
  \item[(2)] $\tau_2\colon (\bar u,t)\mapsto \bar u_{<t}$.
  \end{enumerate}
\end{Prop}
\begin{proof} (1) Suppose $s<|\bar u|<t$ and let $\e=\min\{|\bar u|-s,t-|\bar u|\}$.
    Since $|\bar u-\bar u_1|$ is bounded by $\varrho_{\UU_M}(\bar u_1,\bar u)$
    for all $\bar u_1$, if $\varrho_{\UU_M}(\bar u_1,\bar u)<\e$, then
    also $s<|\bar u_1|<t$. This shows that the inverse image of the open
    interval $(s,t)$ under $\tau_1$ is open.
    % [In the following is a proof that \tau_1 is open, but I decided to drop that statement
    % fromt he theorem. We don't need it and it is difficult to prove for \tau_3 (if even possilbe)]
    % On the other hand assume that $O\subset\UU_M$ is open and $t\in\tau_1[O]$.
    % Let $\bar u\in O$ be such that $\tau_1(\bar u)=|\bar u|=t$. Let $\e>0$
    % be such that $B_{\UU_M}(\bar u,\e)\subset O$. Let $u_1$ be a constant control signal 
    % with $|u_1|=\e$. Then, $\bar u_s=(\bar u\oplus u_1)_{|\bar u|+s\e}\in B_{\UU_M}(\bar u,\e)$
    % for all $-1<s<1$. But $\tau_1(\bar u_s)=\bar u+s\e$, so we have found a neighbourhood 
    % of $t$ in $\tau_1[O]$ showing that it is open.
    
(2) Let $\UU_M\times \R_{\ge 0}$ be equipped with the metric
    $\delta((\bar u,t),(\bar v,s))=\varrho_{\UU_M}(\bar u,\bar v)+|t-s|$. This metric is compatible
    with the product topology; thus, it suffices to show continuity of $\tau_2$ w.r.t.~$\delta$.
    Suppose $(\bar u,t),(\bar v,s)\in\UU_M\times \R_{\ge 0}$. 
    Let $T_1=\min\{|\bar u|,|\bar v|,t,s\}$,
     $T_2=\min\{|\bar u|,|\bar v|\}$,
     $T_3=\min\{|\bar u|,t\}$, and $T_4=\min\{|\bar v|,s\}$.
    Now $T_1\le T_2$, $T_3=|\bar u_{<t}|\le |\bar u|+t$,
    $T_4=|\bar v_{<s}|\le |\bar v|+s$, and so we get:
    \begin{align*}
        \varrho_{\UU_M}(\tau_2(\bar u,t),\tau_2(\bar v,s))&=\varrho_{\UU_M}(\bar u_{<t},\bar v_{<s})\\
        &=\int_0^{T_1} d_U(\bar u_{<t}(z),\bar v_{<s}(z))dz+||\bar u_{<t}|-|\bar v_{<s}||\\
        &=\int_0^{T_1} d_U(\bar u(z),\bar v(z))dz+|T_3-T_4|\\
        &\le\int_0^{T_2} d_U(\bar u(z),\bar v(z))dz+||\bar u|-|\bar v|+t-s|\\
        &\le\int_0^{T_2} d_U(\bar u(z),\bar v(z))dz+||\bar u|-|\bar v||+|t-s|\\
        &=\varrho_{\UU_M}(\bar u,\bar v)+|t-s|\\
        &=\delta((\bar u,t),(\bar v,s)),
    \end{align*}
    which shows that $\tau_2$ is in fact $1$-Lipschitz and therefore continuous.
\end{proof}

\subsection{Space of trajectories and path actions given by differential structures}
\label{ssec:SpaceOfTrajectories}

We first clarify our notion of the state space
and what is its relationship to the ambient space in which the robot
is, and to the notion of an environment defined in Section~\ref{ssec:Life-paths}.
A robot, as a body occupying physical space, can be in various
configurations. The robot is then embedded in some ambient space
which both restricts and extends this set. On the one hand, not only
can the body be in some configuration, but it can also be in different
locations and have different orientations in the ambient space. On the other hand, the ambient space
may restrict the range of possible configurations that the robot's body can achieve.
This gives rise to the state space~$X$. 
For example, consider a car-like robot with four wheels and front steering.
Its body's configuration space is a subset of $(S^1)^{4} \times \R$.
Once embedded in an ambient 3D world $\R^3$ and restricting the car to contact a planar surface, the space of robot's possible
states
$X$ becomes a subset of $(S^1)^4 \times \R \times \R^2 \times S^1$. The third and fourth factors account for the car's position and orientation, respective, in the plane.
Thus, the \emph{ambient space} is the theoretical space in which the robot
is and we do not refer to it in our theory. 
Often, $x \in X$ may also encode configuration velocities and other environmental particulars.
We assume
in this paper that $X$ is a metric space. We obtain the notion of an \emph{environment}
in the next section by equipping $X$ with a sensor mapping and a path action.

A control signal $\bar u\colon [0,T)\to U$ influences the robot's
{\em state} in the state space~$X$.   If $X$ is a smooth manifold,
as in many applications, we can consider a parameterized vector field
$f\colon X\times U\to TX$, where $TX$ is the tangent bundle such that
$f(x,u)\in T_xX$. Then, each $\bar u$, given an initial point
$x_0\in X$, yields a trajectory $\bar x\colon [0,T]\to X$ is
the integral curve satisfying the following:%
% \footnote{If $U$ is finite, this is also known as a \emph{hybrid
%     dynamical system} or \emph{hybrid
%     transition system}~\protect{\cite{henzinger1995s}}.}
$$\bar x(0)=x_0\quad \bar x'(t)=f(\bar u,x(t))\quad\text{ for all }t\in [0,T]$$
If $f$ is continuous, then $\bar x$
will also be continuous. The existence and uniqueness of $\bar x$ 
for locally Lipschitz $f$ follows from the Picard-Lindelöf theorem. 
Thus, the range of this map $\bar u\mapsto \bar x$ is 
included in the space of
all continuous paths $\bar x\colon [0,T]\to X$. Motivated by this, we define:
\begin{equation}
  \label{eq:defofXX}
  \XX=\Cup_{T\in\R_{\ge 0}}C([0,T],X). 
\end{equation}

Again following
\cite{Yershov2010}, and assuming that $X$ is
equipped with the metric $d_X$, we can equip $\XX$ with the metric
\begin{equation}
  \label{eq:RhoXX}
  \varrho_{\XX}(\bar x_1,\bar x_2)=\sup\{d_X(\bar x_1(t),\bar x_2(t))\mid t\in [0,T]\}+|T_2-T_1|,
\end{equation}
where $T=\min\{T_1,T_2\}$. As with $\UU_M$, we show that $\XX$ is a Polish
space.

\begin{Prop}\label{prop:XPolish}
  If $d_X$ is a Polish metric on $X$, then  $\varrho_\XX$ is a
  Polish metric on~$\XX$.
\end{Prop}
\begin{proof}
  The argument for completeness is similar to that of the proof of
  Proposition~\ref{prop:UPolish}. The only difference is that instead
  of the $L_1$-metric we use the sup-metric which is also known to be
  complete (in fact Polish) in this case~\cite[(4.19)]{Kechris1995}. To find
  a countable dense set, again use the fact that
  $\XX_{T}=\{\bar x\in\XX:|\bar x|=T\}$ is Polish for all $T$ and let
  $$D=\Cup_{T\in\Q_+}D_T,$$
  where $D_T$ is a dense countable set of~$\XX_T$ and $T$ ranges
  over positive rationals. As a countable
  union of countable sets $D$ is countable and is easily seen to be
  dense in~$\XX$.
\end{proof}

Denote by $\bar r\colon\UU_M\times X\to\XX$ the map that takes $(\bar u,x)$ to the
corresponding trajectory $\bar x$,
\begin{equation}
  \label{eq:defofbarp}
  \bar r\colon (\bar u,x)\mapsto \bar x .
\end{equation}
Let $r\colon\UU_M\times X\to X$ be the map
$r(\bar u,x)=\bar r(\bar u,x)(|\bar r(\bar u,x)|)$.
It was shown in \cite{Yershov2010} that if
$f$ is Lipschitz and $X$ is a subspace of $\R^n$, then $\bar r$
is continuous. 
We prove a slight strengthening of that:

\begin{Prop}\label{prop:pCont}
  If $f$ is uniformly continuous and $X\subset\R^n$, then both $\bar r$
  and $r$ are continuous.
\end{Prop}
\begin{Remark}
    We will see in Corollary~\ref{cor:ContinuousPathAction} that under very
    general conditions the continuity of $r$ implies the continuity of~$\bar r$.
\end{Remark}
\begin{proof}
  We prove the continuity of $\bar r$. The continuity of $r$ will then
  follow from the continuity of the projection map $\bar x\mapsto \bar x(|\bar x|)$.
  Fix $(\bar u_1,x_1)$ and $(\bar u_2,x_2)$ arbitrarily. Let $T_1=|\bar u_1|$, $T_2=|\bar u_2|$
  and $T=\min\{T_1,T_2\}$.
  Let $\e>0$ and $\delta_1$ be chosen such that $d(f(x_2,u_2),f(x_1,u_1))<\e/(3T)$
  whenever $d_X(x_1,x_2)+d_U(u_1,u_2)<\delta_1$, which exists by the uniform
  continuity of~$f$. Let $\delta=\min\{\e,\delta_1\}/3$.
  Suppose now
  $\varrho_{\UU_M}(\bar u_1,\bar u_2)+|x_2-x_1|<\delta$
  and let $\bar x_1=r(\bar u_1,x_1)$, $\bar x_2=r(\bar u_2,x_2)$. 
  Now,
  \begin{align*}
    &\varrho_{\XX}(\bar x_1,\bar x_2)\\
    &\le\sup\Big\{|x_1-x_0|+\int_0^t \big|f(\bar x_1(t),\bar u_1(t))-f(\bar x_2(t),\bar u_2(t))\big|dt \mid t\in [0,T]\Big\}+|T_2-T_1|\\
    &\le\sup\Big\{\int_0^t \e/(3T) dt \mid t\in [0,T]\Big\}+|x_2-x_1|+|T_2-T_2|\\
    &=\sup\big\{t\e/(3T)  \mid t\in [0,T]\big\}+|x_2-x_1|+|T_2-T_2|\\
    &=\e/3+\delta+\delta\\
    &\le\e/3+\e/3+\e/3\\
    &=\e,
  \end{align*}
  which proves that $r$ is continuous at the arbitrary point~$(\bar u_1,x_1)$.
\end{proof}

\subsection{Topological versions}
\label{sec:Topological}

The purpose of this
section is to free ourselves from differential calculus and enable
more flexible usage of topological machinery. In Definition~\ref{def:PathSpaceBasics}
we required that $U$ is a topological space, because we had to talk about
measurable functions with range~$U$. In the below definition we
do not need to equip $U$ with a topology, although usually in most applications
it has a natural topology that comes with it. 

\begin{Def}\label{def:TopoU}
  Suppose $U$ is any set and let $U^*$ be the set of
  \emph{all} functions $\bar u\colon [0,T)\to U$. Define $|\bar u|$,
  $\lhd$ and $\oplus$ the same way as in
  Definition~\ref{def:PathSpaceBasics}. We define
  $\UU$ to be any subset of $U^*$ which satisfies (2) and (3) of
  Proposition~\ref{prop:UU_Mclosedunder}, that is, 
  \emph{closed under
    segmentation} and \emph{extensive},
    and  that it is equipped with a topology satisfying
  Lemma~\ref{lemma:ProjectionsContinuous}, that is, the projection
  maps $\tau_1\colon \bar u\mapsto |\bar u|$ and $\tau_2\colon (\bar u,t)\mapsto \bar u_{<t}$ are continuous.
\end{Def}

We use Propositions~\ref{prop:UU_Mclosedunder} and~\ref{lemma:ProjectionsContinuous} 
to justify
Definition~\ref{def:TopoU} as a proper generalization of $\UU_M$ of the
previous section. We will use the space $\XX$ the way we already defined
it, see~\eqref{eq:defofXX} and~\eqref{eq:RhoXX}, assuming $X$ is a Polish metric space
(we no longer assume that it is a manifold). 

% [It was too hard for me to spontaneously prove the following so commented out]
% We give
% one more fact to justify its robustness in a topological sense
% (see~\cite{Kechris1995} for a methodology to prove it):

% \begin{Lemma}
%   The Polish (by Proposition~\ref{prop:XPolish}) topology on $\XX$
%   generated by $\varrho_{\XX}$ is independent on the choice of the Polish
%   metric $d_X$ on $X$ as long as $d_X$ generates the same topology
%   on~$X$. 
% \end{Lemma}
% \begin{proof}[Sketch of a proof]
%     The topology on $C([0,T],X)$ given by the $\sup$-metric is independent on 
%     $d_X$ as long as $d_X$ generates the same Polish topology on~$X$~\cite[§4.E]{Kechris1995}.
%     But the topology on $\XX$ given by $\rho_\XX$ is determined by the topology
%     on each $C([0,T],X)$, $T\in\R_{\ge 0}$, by the observation that it 
%     is generated by the open sets of the form
        
%     is the finest
%     topology such that all the inclusions $C([0,T],X)\hookrightarrow \XX$
%     as well as the map $\XX\to \R$ defined by $\bar x\mapsto |\bar x|$ are all continuous.
% \end{proof}

Finally, here is the definition of path action which is independent
of a differential, or other non-topological, structure on~$X$ or~$U$:

\begin{Def}\label{def:path-action}
  A \emph{path action} of $\UU$ on $X$ is a
  continuous function
  $p\colon \UU\times X\to X$ such that
  \begin{description}
  \item[(PA1)] $p(\es,x)=x$ for all $x\in X$, and
  \item[(PA2)] $p(\bar u,x)=p(\bar u_{\ge t},p(\bar u_{<t},x))$
    for all $x\in X$, all $\bar u\in \UU$ and all $t\in\R_{\ge 0}$.
  \end{description}
\end{Def}

\begin{Remark}\label{remark:ClosedUnderOplus}
  Perhaps a more natural definition of an action would to have the following clause instead
  of $(PA2)$:
  \begin{description}
  \item[(PA2$^\prime$)] $p(\bar u_0\oplus \bar u_1,x)=p(\bar u_1,p(\bar u_{0},x))$
    for all $x\in X$, all $\bar u_0,\bar u_1\in \UU$.
  \end{description}
  This is equivalent to (PA2) when $\UU$ is closed under $\oplus$.
  The space $\UU_M$ is such (Proposition~\ref{prop:UU_Mclosedunder}).
  Suppose, however, one wanted to consider $\UU_C\subset \UU_M$
  which consists only of continuous paths. Then, $\UU_C$ is not closed
  under $\oplus$: if
  $\lim_{t\to |\bar u_0|}\bar u_0(t)\ne \bar u_1(0)$, then
  $\bar u_0\oplus \bar u_1$ is not continuous. In this case, clause $(PA2)$ comes in handy as it serves
  the same role as $(PA2')$ but does not require closure
  under~$\oplus$.  It \emph{does}, however, require that the space is
  closed under segmentation (Definition~\ref{def:TopoU}).
  The choice between $(PA2)$ and $(PA2')$ will not be important in
  this paper until Theorem~\ref{thm:Bisimulation}.
\end{Remark}

Proposition~\ref{prop:pCont} justifies the assumption of continuity in Definition
\ref{def:path-action}. Here $p$ corresponds to $r$. If we use Proposition~\ref{prop:pCont}
as a justification for the generalization, the reader may wonder why
we did not additionally assume the continuity of the induced map into trajectories defined by
\begin{equation}
  \label{eq:defofbarpgeneral}
  \bar p(\bar u,x)\colon [0,|\bar u|]\to X\qquad \bar p(\bar u,x)(t)=p(\bar u_{<t},x)\qquad 
 0\le t\le |\bar u|.
\end{equation}
This is because continuity is implied by Corollary~\ref{cor:ContinuousPathAction} below.

\begin{Lemma}\label{lemma:supCont}
    Let $Z$ be a topological space and let $f\colon Z\times [0,T]\to \R$
    and $h\colon Z\to [0,T]$
    be continuous. Then, the function $g\colon Z\to \R$ defined by 
    $$g(z)=\sup\{f(z,t)\mid t\in [0,h(z)]\}$$
    is continuous.
\end{Lemma}
\begin{proof}
    Let $a<b$ be real numbers. We will show that $g^{-1}(a,b)$ is open in~$Z$.
    Let 
    \begin{align*}
      E&=\{z\mid \exists t\in [0,T](t\le h(z)\land f(z,t)>a)\},\text{ and}\\
      A&=\{z\mid \forall t\in [0,T](t\le h(z)\rightarrow f(z,t)<b)\}.
    \end{align*}
    By the continuity of $h$ and $f$, we can replace
    ``$\le$'' by ``$<$'' in the definition of $E$, so
    $$E=\{z\mid \exists t\in [0,T](t< h(z)\land f(z,t)>a)\}.$$
    We now have $g^{-1}(a,b)=E\cap A$; thus, it suffices to show that $E$ and $A$ are open.
    The set $E$ is the projection of $\{(z,t)\mid t<h(z)\}\cap f^{-1}(a,\infty)$ which 
    is open by the continuity of $h$ and~$f$.
    Thus, it remains to show that $A$ is open. 
    Let $z_0\in A$. We will find an open neighborhood
    $O$ of $z_0$ with $O\subset A$.
    Let $O^b=f^{-1}(-\infty,b)$. Since it is open by the continuity of $f$, for each $t\in [0,h(z_0)]$
    we have $(z_0,t)\in O^b$ and it has a rectangular open neighbourhood $O^t\times I^t$ in $O^b$:
    \begin{equation}
        \label{eq:OtIt}
        (z_0,t)\in O^t\times I^t\subset O^b.
    \end{equation}
    By compactness, find $t_0,\dots,t_n$ such that $I^{t_0},\dots,I^{t_n}$ cover $[0,h(z_0)]$.
    Let $O=O^{t_0}\cap\cdots\cap O^{t_n}$. Clearly, $z_0\in O$. It remains to show
    that $O\subset A$. So let $z\in O$ and $t\in[0,h(z_0)]$ be arbitrary.
    Let $k$ be such that $t\in I^{t_k}$. However, then 
    $(z,t)\in O^{t_k}\times I^{t_k}{\subset}O^b$ by \eqref{eq:OtIt};
    thus, $f(z,t)<b$ by the definition of~$O^b$.
\end{proof}

Below, let $X^*$ be the set of all functions $[0,T]\to X$, $T\in\R_{\ge 0}$.

\begin{Prop}\label{prop:ContinuousGen}
  Let $\UU$ be as in Definition~\ref{def:TopoU}, $Z$
  any topological space, $(X,d_X)$ a Polish metric space,
  and $\XX$ as in~\eqref{eq:defofXX} with metric as in \eqref{eq:RhoXX}.
  Suppose $p\colon\UU\times Z\to X$ is continuous and define
  the function $\bar p\colon\UU\times Z\to X^*$ as in \eqref{eq:defofbarpgeneral}, by
  \begin{equation}
  \label{eq:defofbarpverygeneral}
    \bar p(\bar u,z)\colon [0,|\bar u|]\to X\qquad \bar p(\bar u,x)(t)=p(\bar u_{<t},x)\qquad 
   0\le t\le |\bar u|.
  \end{equation}
  Then, the range of $\bar p$ is a subset of $\XX$, and $\bar p$ is continuous.
\end{Prop}
\begin{proof}
  To check that $\bar x=\bar p(\bar u,x)$ belongs to $\XX$,
  simply note that by the continuity of $p$ and of $\bar u\mapsto \bar u_{<t}$,
  $\bar x$ is also continuous.  
  For the continuity of $\bar p$, it is enough
  to show that for all $(\bar u_0,z_0)\in \UU\times Z$ and all $\e$,
  the inverse image of $B_\XX(\bar p(\bar u_0,z_0),\e)$ is open.
  Thus, fix $(\bar u_0,z_0)\in \UU\times Z$. Let $T=|u_0|$.
  Define $f\colon \UU\times Z\times[0,T]\to \R$
  by
  $$f(\bar u,z,t)=d_X(p(\bar u_{<t},z),p((\bar u_0)_{<t},z_0))+|T-|\bar u||.$$
  By the continuity of $(\bar u,t)\mapsto \bar u_{<t}$, of $\bar u\mapsto |\bar u|$
  (see Definition~\ref{def:TopoU}),
  of $p$, and of the metric $d_X$, $f$ is continuous. Let $h\colon\UU\times Z\to [0,T]$
  be defined by $h(\bar u,z)=|\bar u|$ which is again continuous. By Lemma~\ref{lemma:supCont},
  the function $g\colon \UU\times Z\to \R$ given by $g(\bar u,z)=\sup\{f(\bar u,z,t)\mid t\in [0,h(\bar u,z)]\}$
  is continuous. But
  $g(\bar u,z)=\varrho_\XX(\bar p(\bar u,z),\bar p(\bar u_0,z_0))$. Now consider the inverse image
  \begin{align*}
    \bar p^{-1}B_\XX(\bar p(\bar u_0,z_0),\e)&=\{(\bar u,z)\in\UU\times Z\mid g(\bar u,z)<\e\}\\
    &=g^{-1}(-\e,\e).
  \end{align*}
  which is open by the continuity of~$g$.
\end{proof}

\begin{Cor}\label{cor:ContinuousPathAction}
  Suppose $p\colon\UU\times X\to X$ is a path action.
  Let $\bar p$ be defined as in \eqref{eq:defofbarpgeneral}.
  Then, $\bar p\colon \UU\times X\to \XX$ is continuous.
\end{Cor}
\begin{proof}
    Choose $Z=X$ in Proposition~\ref{prop:ContinuousGen}.
\end{proof}

For our purposes a full-fledged action as described in
Definition~\ref{def:path-action} is often not necessary. In our setup
(Definition~\ref{def:Environment}) the environment will always have a
unique initial state $x_0\in X$ for the robot; thus, all trajectories
will start from that point. It is only a technicality, as switching
the initial state can be formalized as switching the environment from
$(X,x_0)$ to $(X,x_1)$. However, it will make mathematics easier for our
considerations of covering spaces. Because of this, it will often
be enough to consider the restriction $p\rest (\UU\times \{x_0\})$.
Thus, we define:

\begin{Def}\label{def:initialized_path-action}
  Let $(X,x_0)$ be a pointed space. An \emph{initialized path action}
  is a continuous $p\colon \UU\to X$ such that $p(\es)=x_0$. We also
  denote $p\colon\UU\to (X,x_0)$, or even
  $p\colon(\UU,\es)\to (X,x_0)$, to emphasize that it is a map between
  pointed spaces. As for path actions \eqref{eq:defofbarpgeneral}, given an initialized
  path action $p$, define 
  $\bar p\colon \UU\to \XX$ by
  \begin{equation}
    \label{eq:defofbarpgeneral_init}
    \bar p(\bar u)(t)=p(\bar u_{<t}), \quad t\le |\bar u|.
  \end{equation}
\end{Def}

\begin{Remark}\label{rem:FromPathTOInitPath}
  Given a path action $p\colon\UU\times X\to X$, the function
  $p_{x_0}\colon\UU\to X$ defined by $p_{x_0}(\bar u)=p(\bar u,x_0)$
  is an initialized path action.
\end{Remark}

\begin{Prop}\label{prop:ContinuousInitPA}
  Let $p\colon \UU\to (X,x_0)$ be an initialized path action.
  Then, $\bar p$ is continuous.
\end{Prop}
\begin{proof}
  Choose $Z=\{x_0\}$ in Proposition~\ref{prop:ContinuousGen}, and identify
  $\UU\times \{x_0\}$ with~$\UU$.
\end{proof}

\section{History Information Spaces}
\label{sec:HISS}

In this section we first define environments and how
trajectories in $X$ become also trajectories in the sensory
space (Section~\ref{ssec:Life-paths}). Then we generalize
history information spaces in Section~\ref{ssec:ContHIS}.

\subsection{Trajectories in environments}
\label{ssec:Life-paths}

We want to now make precise the idea that two environments are indistinguishable from the point of view of
some mobile robot. We introduce some more definitions.  Fix a
Polish space of control signals $\UU$ as in
Definition~\ref{def:TopoU}.  Let the {\em observation space} $Y$ be any 
topological space, corresponding to the set of all outputs of a given 
sensor connected to the robot. When comparing environments, they all should
share $\UU$ and $Y$ because these are the ``interfaces'' between the robot
and its environment. If they are different, it means that the robot
has different actuators or different sensors, and comparing such
situations is beyond our present analysis. Thus, we consider that $\UU$
and $Y$ are fixed for the rest of the paper.

The algebra of Borel sets in a Polish space is the smallest
$\sigma$-algebra containing the basic open sets. A function from a
Polish space to another is \emph{Borel} if the inverse image of every
open set is Borel. Equivalently, the inverse image of every Borel set
is Borel. 

\begin{Def}\label{def:Environment}
  An \emph{environment} is a tuple $E=(X,x_0,h,p)$ where $X$ is a
  Polish space, $x_0\in X$ is the \emph{initial position},
  $h\colon X\to Y$ a Borel \emph{sensor mapping} \cite{LaValle2006}, and $p$ is either
  an initialized path action $p\colon \UU\to X$ or a path action
  $p\colon\UU\times X\to X$ (Definitions~\ref{def:path-action}
  and~\ref{def:initialized_path-action}).  The pair $(h,p)$ is called
  a \emph{sensorimotor structure} on $(X,x_0)$. We will assume that $p$
  is an initialized path action unless mentioned otherwise.
\end{Def}

We require $h$ to be Borel because continuity is too strong of a
requirement in general, but to have no requirements at all would make
working with $h$ difficult. The class of Borel functions is loose
enough to include all functions that are generally interesting in this
context, such as piecewise continuous functions. Unlike measurable
sets, Borel sets are topologically invariant (preserved by
homeomorphisms). If we were to define a Radon measure on $X$,
then all Borel functions would automatically be measurable. Moreover,
all Borel functions are continuous on a co-meager set. (Co-meager sets
are sets containing an intersection of countably many dense open sets.) Yet another
benefit for us is that the composition of Borel functions 
is a Borel function and hence measurable. 

Let $\YY$ be the space of measurable functions
$\bar y\colon [0,T)\to Y$, just as $\UU$ is the set of measurable
functions into~$U$ (Definition~\ref{def:PathSpaceBasics}).
As defined in \eqref{eq:defofbarpgeneral_init}, each path
$\bar u\colon [0,T)\to U$ in $\UU$ generates a path
$\bar p(\bar u)=\bar x_{\bar u}\colon [0,T]\to X$ defined by
\begin{equation}
  \label{eq:defofx_baru}
  \bar x_{\bar u}(t)=p(\bar u_{<t})=\bar p(\bar u)(t)
\end{equation}
for all $0\le t<T$.  This
$\bar x_{\bar u}$ is the trajectory that the robot will traverse in
the configuration space $X$ starting from its initial position $x_0$
and applying the control given by~$\bar u$. This trajectory generates a unique path
$\bar y_{\bar u}\colon \R_{\ge 0}\to Y$ in the robot's observation space
defined by
\begin{equation}
  \label{eq:DefOfysubu}
  \bar y_{\bar u}(t)=h(\bar x_{\bar u}(t))=h(\bar p(\bar u)(t))=h(p(\bar u_{<t})).
\end{equation}
Since $\bar x_{\bar u}$ is continuous and $h$ is Borel,
$\bar y_{\bar u}$ is measurable.

Define the induced sensory trajectory map
$\bar h\colon\XX\to \YY$ by
\begin{equation}
  \label{eq:defofbarh}
  \bar h(\bar x)(t)=h(\bar x(t)).
\end{equation}
Using this notation, we have
\begin{equation}
  \label{eq:ybar_in_terms_of_barhbarp}
  y_{\bar u}=\bar h(\bar p(\bar u)).
\end{equation}

\subsection{Continuous-time history information states}
\label{ssec:ContHIS}

The following definition can be thought of as a continuous-time version of
the history information space $\II$ \cite[Ch.~11]{LaValle2006},
and~\cite{Sakcak2023,Tovar2005,WeiSakLav22}.

\begin{Def}[History information space]
  A pair $\eta=(\bar u,\bar y)\in \UU\times \YY$ is a \emph{history
    information state}.  The set of all information states
  $\UU\times \YY$ is denoted by~$\II$.
\end{Def}

Now, each environment $E=(X,x_0,h,p)$ carves out a subspace of $\II$
which contains only those histories that are possible in~$E$.

\begin{Def}\label{def:IIsupE}
  Given an environment $E$, let $\II^{E}$ be the set of all pairs
  $(\bar u,\bar y)$ where $\bar y=\bar y_{\bar u}$.  To
  specify in which environment $\bar y$ was obtained from $\bar u$,
  we denote $\bar y^E_{\bar u}=\bar y_{\bar u}$.
\end{Def}

\section{Indistinguishability of Environments}
\label{sec:Indistinguishability}

In this section we deal with the equivalence relations of
indistinguishability of environments. We will later introduce the
non-symmetric relation of {\em strong} indistinguishability in
Section~\ref{ssec:Strongly}.

\subsection{Main construction}

We will now present the main formal ingredients of the theory of
indistinguishability.

\begin{Def}\label{def:I-equiv}
  Two environments $E=(X,x_0,h,p)$ and $E'=(X',x'_0,h',p')$ are
  \emph{$\II$-equivalent}, if for all $\bar u\in \UU$,
  $\bar y^{E}_{\bar u}=\bar y^{E'}_{\bar u}$.  Denote the
  $\II$-equivalence relation by~$\equiv^{\II}$.
\end{Def}

We have the following:
\begin{Lemma}\label{lemma:TFAE}
  The following are equivalent:
  \begin{enumerate}
  \item $E\equiv^{\II} E'$,
  \item $\II^{E}=\II^{E'}$,
  \item $\bar h\circ \bar p=\bar h'\circ \bar p'$,
  \item $h\circ p=h'\circ p'$.
  \end{enumerate}
\end{Lemma}
\begin{proof}
  \begin{description}
  \item[1.$\Rightarrow$2.] follows from the definitions of $\equiv^{\II}$
    and $\II^{E}$, $\II^{E'}$.
  \item[2.$\Rightarrow$3.] Suppose $\bar u\in \UU$.  Then, there is only
    one $\bar y\in \YY$ such that the pair $(\bar u,\bar y)$ is in
    $\II^{E}$ and it is $\bar y=\bar y^{E}_{\bar u}$ which by
    \eqref{eq:ybar_in_terms_of_barhbarp} equals
    \begin{equation}
      \label{eq:yEuhpu}
      \bar y^{E}_{\bar u}=\bar h(\bar p(\bar u)).
    \end{equation}
    However, since
    $\II^{E}=\II^{E'}$, also
    $(\bar u,\bar y^{E}_{\bar u})\in \II^{E'}$. By the fact that
    in $\II^{E'}$ there is only one pair with the first coordinate
    equal to $\bar u$, we have $\bar y^{E}_{\bar u}=\bar y^{E'}_{\bar u}$
    and using \eqref{eq:ybar_in_terms_of_barhbarp} again we have
    $\bar y^{E'}_{\bar u}=\bar h'(\bar p'(\bar u))$
    and using \eqref{eq:yEuhpu} we have
    $\bar h(\bar p(\bar u))=\bar h'(\bar p'(\bar u))$.
    By the arbitrary choice of $\bar u$, we conclude~3.
  \item[3.$\Rightarrow$4.] Let $\bar u\in \UU$.  
    By~\eqref{eq:defofbarpgeneral_init}
    we can write
    \begin{equation}
      \label{eq:baruupu}
      \bar p(\bar u')(|\bar u|)=p(\bar u)\quad\text{ and }\quad\bar p'(\bar u')(|\bar u|)=p'(\bar u).
    \end{equation}
    Using \eqref{eq:defofbarh}, we now obtain:
    \begin{align*}
      h(p(\bar u))&\stackrel{\eqref{eq:baruupu}}{=}h(\bar p(\bar u')(|\bar u|)) 
                    \stackrel{\eqref{eq:defofbarh}}{=}\bar h(\bar p(\bar u'))(|\bar u|)
                    =(\bar h\circ\bar p)(\bar u')(|\bar u|)\\
                  &\stackrel{3.}{=}(\bar h'\circ\bar p')(\bar u')(|\bar u|) 
                  =\bar h'(\bar p'(\bar u'))(|\bar u|) 
                  \stackrel{\eqref{eq:defofbarh}}{=}h'(\bar p'(\bar u')(|\bar u|)) \\
                  &\stackrel{\eqref{eq:baruupu}}{=}h'(p'(\bar u)).
    \end{align*}
  \item[4.$\Rightarrow$1.] Let $\bar u\in \UU$. We need to show that
    $\bar y_{\bar u}^{E}=\bar y^{E'}_{\bar u}$. By~\eqref{eq:DefOfysubu},
    for all $t<|\bar y_{\bar u}|$ we have
    $$\bar y^E_{\bar u}=h(p(\bar u_{<t}))\stackrel{4.}{=}h(p(\bar u_{<t}))=\bar y^{E'}_{\bar u},$$
    which completes the proof.
  \end{description}
\end{proof}

\begin{Cor}\label{cor:ER}
  $\equiv^{\II}$ is an equivalence relation.
\end{Cor}
\begin{proof}
  Follows easily from any of the characterizations given by
  Lemma~\ref{lemma:TFAE}.
\end{proof}

Thus, $\II$-equivalence means that no matter what the robot \emph{does}, it cannot
receive different sensory readings in these two environments. We
return to the example of a circle and a line.

\begin{Example}\label{ex:CircleAdvanced}
  Fix $U=\{-1,1\}$ and some Polish observation space~$Y$. Suppose $X=S^1$ is given
  as $\{e^{i\theta}\mid \theta\in\R\}$ with the initial point
  $x_0=e^{0}$. Suppose $h\colon S^1\to Y$ is a continuous sensor
  mapping and given $\bar u\in\UU$, the robot's state is given by
  $p(\bar u)=e^{i\theta(\bar u)}$, where
  $$\theta(\bar u)=\int_{0}^{|\bar u|} \bar u(t)dt.$$
  This defines an environment $E=(X,x_0,h,p)$.

  Now let $X'=\R$, $x_0'=0$, and
  $p'(\bar u)=\int_{0}^{|\bar u|} \bar u(t)dt$.  To define
  $h'\colon\R\to Y$, let $f\colon\R\to S^1$ be the covering map
  $t\mapsto e^{it}$, and let $h'=h\circ f$. Now $p'$ is a lifting
  of $p$, $f\circ p'=p$, and the following diagram commutes:
  \begin{equation}
    \label{eq:Commutative1}
    \begin{tikzcd}
      & (X',x'_0) \arrow[d,"f"]  \arrow[rd,"h'"]& \\
      (\UU,\es)  \arrow[r, "p"]\arrow[ru, "p'"]& (X,x_0) \arrow[r,"h"]  & Y.
    \end{tikzcd}
  \end{equation}
  This diagram shows \emph{pointed spaces} which are pairs
  $(Z,z)$ with $Z$ a space and $z\in Z$ a point. The arrows
  correspond to maps which take the selected point to the selected
  point. The selected point in $\es \in \UU$ is the empty control signal
  corresponding to~$T=0$. We have not shown the selected point in $Y$
  because our maps do not have a requirement to map $x_0$ or $x'_0$ to
  any particular point, although due to commutation, we know that
  $h'(x_0')=h(x_0)$. From the above we have
  $$h\circ p=h\circ (f\circ p')=(h\circ f)\circ p'=h'\circ p'.$$
  Thus, by Lemma~\ref{lemma:TFAE}($1.\Leftrightarrow 4$),
  we have $E'\equiv^{\II}E$, and so the environments are indistinguishable.

  Note that we \emph{defined} $h'$ from $h$ using $f$. In
  Section~\ref{sec:Covering} we will also see that such $p'$ as above
  can always be obtained from any initialized path action~$p$.  This
  means that \emph{no matter which} sensor mapping there is on the
  circle, the possibility that it is actually the line can never be
  ruled out.
  We will see in Section~\ref{sec:Covering} (especially
  Theorem~\ref{thm:Lift}) that the notion of a covering space plays a
  key role here.
\end{Example}

We can also formulate an equivalence in terms of ``eternal'' or unbounded
trajectories. This will have the advantage of increased generality. We define:

\begin{Def}\label{def:Branches}
  A \emph{branch} through $\UU$ is defined to be a function
  ${\bf u}\colon \R_{\ge 0}\to U$ such that for all $t\in\R_{\ge 0}$,
  we have that ${\bf u}\rest [0,t)\in \UU$. Let $B\UU$ be the set of
  branches through~$\UU$. Similarly, denote by ${\bf x}$ a branch
  through $\XX$, i.e., a function ${\bf x}\colon \R_{\ge 0}\to X$ such
  that ${\bf x}\rest [0,T]\in \XX$ for all $T\in \R_{\ge 0}$. Denote
  the set of branches through $\XX$ by $B\XX$. Similarly let $B\YY$ be
  the set of branches through $\YY$ defined analogously. 
  \end{Def}

  Denote by ${\bf p}$ and ${\bf h}$ the natural extensions of $\bar p$ and
  $\bar h$ to the sets of branches:
  \begin{equation}
    \label{eq:Defofbfp}
    {\bf p}\colon B\UU\to B\XX,\qquad {\bf p}({\bf u})(t)=p({\bf u}_{<t}),
  \end{equation}
  \begin{equation}
  \label{eq:Defofbfh}
  {\bf h}\colon B\XX\to B\YY,\qquad {\bf h}({\bf x})(t)=h({\bf x}(t)),
  \end{equation}
  in which ${\bf u}_{<t}={\bf u}\rest [0,t)$ is defined just like for
  elements of~$\UU$.

For the bounded-domain paths, given a path action
$p\colon \UU\times X\to X$, each branch ${\bf u}$ through $\UU$
generates a unique trajectory
${\bf x}_{{\bf u}}={\bf p}({\bf u})\colon\R_{\ge 0}\to X$ through the
state space given by
\begin{equation}
  \label{eq:def_bfx}
  {\bf x}_{{\bf u}}(t)=p({\bf u}_{<t}) .
\end{equation}
This trajectory generates a unique path
${\bf y}_{\bf u}={\bf h}({\bf x})\colon \R_{\ge 0}\to Y$ in the observation space defined by
${\bf h}({\bf x})(t)=h(p({\bf u}\rest t))$.  We write
${\bf y}^{E}_{{\bf u}}$ to specify the environment $E$ in which it was
computed.

\begin{Def}\label{def:fulllife}
  The branches ${\bf x}\in B\XX$ are called \emph{full trajectories} and
  ${\bf y}\in B\YY$ are called the \emph{full sensory histories}. The
  set of pairs $({\bf u},{\bf y})\in B\UU\times B\YY$ is the set of
  \emph{full information histories} and we denote it by
  $\mathbfcal{I}$. Analogously to Definition~\ref{def:IIsupE}
  we also define $\mathbfcal{I}^E\subset \mathbfcal{I}$ to
  be the set of the pairs of the form~$({\bf u},{\bf y}^E_{\bf u})$
  for a given environment~$E$.
\end{Def}

We can now reformulate our definition of equivalence in terms of full
information histories. This definition is more general than the
definition of $\II$-equivalence
(Definition~\ref{def:I-equiv}) in the sense that it includes it
as a special case (see Section~\ref{ssec:EE}), but also gives the
possibility of defining a whole class of filter-based equivalence relations. A \emph{filter} is an equivalence relation $F$
either on $Y$, on $\YY$, or on $B\YY$ usually so that an equivalence on
$Y$ induces equivalences on the function spaces through pointwise
application. Examples of filters are the gap-sensor
(Example~\ref{ex:GapNavTrees}), the beam-sensor
(Example~\ref{ex:Beams2}) or considering histories up to homeomorphisms
(see Section~\ref{ssec:ExamplesToppo}).  

\begin{Def}\label{def:LifePathEq}
  Given an equivalence relation $F$ on the set of full information histories
  $B\UU\times B\YY$, let $\equiv_{F}$ be an equivalence relation on the set of
  all environments $E$, $E'$ such that
  $$E \equiv_F E'$$
  if and only if for all ${\bf u} \in B\UU$,
  $(({\bf u},{\bf y}^{E}_{\bf u}),({\bf u},{\bf y}^{E'}_{\bf u}))\in F$.
  We call this \emph{filter based full historical equivalence}
  or just \emph{filter based equivalence}.
  If $F$ is the identity relation, we drop it from the notation,
  so $\equiv$ is the same as $\equiv_{\id}$. We call this
  \emph{full historical equivalence} or just \emph{historical equivalence}.
\end{Def}

Similarly as for $\equiv^{\II}$, we have:
\begin{Lemma}\label{lemma:TFAE2}
  The following are equivalent:
  \begin{enumerate}
  \item $E\equiv E'$,
  \item $\mathbfcal{I}^{E}=\mathbfcal{I}^{E'}$,
  \item ${\bf h}\circ {\bf p}={\bf h}'\circ {\bf p}'$,
  \item $\bar h\circ \bar p=\bar h'\circ \bar p'$,
  \item $h\circ p=h'\circ p'$,
  \end{enumerate}
\end{Lemma}
\begin{proof}
  The implications $1.\Rightarrow2.\Rightarrow3.$ are proved
  similarly as the same ones in Lemma~\ref{lemma:TFAE},
  and the implications $4.\Rightarrow5.\Rightarrow1.$
  similarly as $3.\Rightarrow4.\Rightarrow1.$ in 
  Lemma~\ref{lemma:TFAE}. The only one requiring a new argument
  is $3.\Rightarrow4.$ which we do now.
  \begin{description}
  \item[3.$\Rightarrow$4.] Suppose $\bar u\in \UU$. Then, by
  the property of being extensive (Definition~\ref{def:TopoU}),
  we can find ${\bf u}\in B\UU$ with ${\bf u}_{<|\bar u|=\bar u}$.
  Let $t\le |\bar u|$. Then,
  \begin{align*}
  \bar h(\bar p(\bar u))(t)&\stackrel{\eqref{eq:defofbarh}}{=} h(\bar p(\bar u)(t))\stackrel{\eqref{eq:defofbarpgeneral_init}}{=} h(p(\bar u_{<t}))
  =h(p(({\bf u}_{<|\bar u|})_{<t}))\\
  &\stackrel{t\le |\bar u|}{=} h(p({\bf u}_{<t}))\stackrel{\eqref{eq:Defofbfp}}{=} h(p({\bf u})(t))\stackrel{\eqref{eq:Defofbfh}}{=} {\bf h}\circ {\bf p}(t)\\
  &\stackrel{3.}{=} {\bf h'}\circ{\bf p'}(t)\stackrel{\eqref{eq:Defofbfh}}{=} h'(p'({\bf u})(t)) \stackrel{\eqref{eq:Defofbfp}}{=}  h'(p'({\bf u}_{<t}))\\
  &= h'(p'(({\bf u}_{<|\bar u|})_{<t}))=  h'(p'(\bar u_{<t})) \stackrel{\eqref{eq:defofbarh}}{=}  \bar h'(\bar p'(\bar u))(t).\qedhere
  \end{align*}  
  \end{description}
\end{proof}

\subsection{Examples of filter-based equivalence relations}
\label{ssec:ExamplesToppo}

\begin{Example}\label{ex:Toppo}
  Define
  $({\bf u},{\bf y}^{E}_{\bf u})\homeo ({\bf u},{\bf y}^{E'}_{\bf
    u})$, if and only if there is a homeomorphism
  $f\colon \R_{\ge 0}\to\R_{\ge 0}$ such that
  ${\bf y}\circ f={\bf y'}$ and ${\bf u}\circ f={\bf u'}$.  Then
  consider $\equiv_{\homeo}$.  Clearly, $\equiv$-equivalence implies
  $\equiv_{\homeo}$-equivalence by choosing the homeomorphism to be
  the identity.  The relation $\equiv_{\homeo}$ pertains to robots
  whose time perception is only relational: there is memory and
  knowledge of which chronological order the events occurred in, but
  not of how much time passed between them. This is because the set of
  homeomorphisms $\R_{\ge 0}\to\R_{\ge 0}$ is exactly the set of
  strictly order preserving bijections.
\end{Example}

\begin{Example}\label{ex:InducedF}
  Suppose $F$ is an equivalence relation on $Y$ making some
  observations indistinguishable. It induces an equivalence relation
  $\bar F$ on $\YY\cup B\YY$ through pointwise application:
  $(\bar y_0,\bar y_1)\in \bar F$ if $\dom (\bar y_0)=\dom(\bar y_1)$
  and for all $t\in\dom(\bar y_0)$ we have
  $(\bar y_0(t),\bar y_1(t))\in F$. Then, $\equiv_{\bar F}$ is an
  equivalence relation which equates those environments which cannot
  be distinguished by any robot that is equipped with this
  filter~$F$. We will abuse the notation and denote $\equiv_F$ in this
  case instead of $\equiv_{\bar F}$ and call it 
  \emph{filter based equivalence induced by~$F$}.
\end{Example}

\begin{Example}\label{ex:Sensorimotor}
  Suppose the robot has a high-level sensor which reports the result
  of a low-level sensorimotor interaction. For example repeatedly
  pushing against a wall can provide the data of how hard the wall is
  and moving whiskers can generate data about the presence of
  obstacles or texture of surfaces~\cite{Diamond2008}.  This can be
  expressed by defining a labeling on the set of information histories
  of fixed length $g\colon \{(\bar u,\bar y)\in\II:|\bar y|=1\}\to L$,
  where $L$ is some set of labels. This induces an equivalence
  relation on $B\YY$ as follows:
  $$({\bf u},{\bf y})\in F_g\iff \forall n\in\N \Big(g({\bf u}\rest [n,n+1))=g({\bf y}\rest [n,n+1))\Big),$$
  or as a moving window as follows:
  $$({\bf u},{\bf y})\in F'_g\iff \forall t\in\R_{\ge 0}\Big( g({\bf u}\rest [t,t+1))=g({\bf y}\rest [t,t+1))\Big).$$  
  Then, $\equiv_{F_g}$ or $\equiv_{F'_g}$ capture the equivalence with
  respect to this filtering. 
\end{Example}

\begin{Open}\label{Q:Ispaces}
    How to generalize the notion of \emph{derived information spaces}
    \cite{LaValle2006,Tovar2005} and of
    \emph{sufficient equivalence relations}~\cite{WeiSakLav22}
    to the continuous framework?
\end{Open}

\subsection{Equivalence of equivalences}
\label{ssec:EE}

Both equivalence relations $\equiv^{\II}$ and $\equiv$ are defined in
a natural way of what would it mean for two environments to be
indistinguishable. In fact, they turn out to be
the same:

\begin{Thm}\label{thm:EquivEquiv}
  For all environments $E$ and $E'$ we have $E\equiv^{\II}E'$ iff
  $E\equiv E'$.
\end{Thm}
\begin{proof}
  Apply Lemma~\ref{lemma:TFAE}(4) and Lemma~\ref{lemma:TFAE2}(5).
\end{proof}

In view of Theorem~\ref{thm:EquivEquiv} it is enough to only talk about
equivalence relations of the form $\equiv_F$ for some equivalence
relation $F$ either on $Y$, on $\YY$ or on $B\YY$ (see
Examples~\ref{ex:Sensorimotor} and~\ref{ex:InducedF}).
This opens an intriguing avenue for future research, but in this paper
we only focus on~$\equiv$ with $F$ being the identity.

We have now defined a class of equivalence relations on tuples
$(X,x_0,h,p)$ which are, for all intents and purposes, continuous
time dynamical systems with a robotic twist (the sensor mapping $h$
and the nature of the action~$p$). The equivalence relations defined
are motivated by robotics because they are based on comparing orbits
(robot's trajectories) in terms of the sensor mapping~$h$. We now state
another open problem:

\begin{Open}\label{Q:Complexity}
  What is the complexity of various $\equiv_F$'s in terms of
  descriptive complexity and Borel reducibility~\cite{Gao2008}?  How
  do they compare to other known equivalence relations on dynamical
  systems? 
\end{Open}

\section{Homomorphisms and Covering Spaces}
\label{sec:HomoCover}

Upon seeing an equivalence relation, it is natural to ask 
\emph{which maps witness it}?
For example, a homeomorphism witnesses homotopy equivalence
but does not witness isometric equivalence.

\subsection{Homomorphisms of environments}

\begin{Def}\label{def:Homomorphism}
  Given environments $E=(X,x_0,h,p)$ and $E'=(X',x_0',h',p')$, a map
  $f\colon X'\to X$ is a \emph{homomorphism} (also denoted as
  $f\colon E'\to E$), if the following conditions are satisfied:
  \begin{description}
  \item[(HOM1)] $f(x'_0)=x_0$,
  \item[(HOM2)] for all $\bar u\in \UU$,
    $f(p'(\bar u))=p(\bar u)$,
  \item[(HOM3)] for all $x\in X$, $h'(x)=h(f(x))$.
  \end{description}
  This is to say that $f$ is such that the following diagram, which
  is already familiar from
  Example~\ref{ex:CircleAdvanced}\eqref{eq:Commutative1}, commutes:
  \begin{center}
    \begin{tikzcd}
      & (X',x'_0) \arrow[d,"f", dashed]  \arrow[rd,"h'"]& \\
      (\UU,\es)  \arrow[r, "p"]\arrow[ru, "p'"]& (X,x_0) \arrow[r,"h"]  & Y.
    \end{tikzcd}
  \end{center}
  If $p$ is a initialized path action (see Definition~\ref{def:Environment}), 
  then we require
  a strengthening of (HOM2):
  \begin{description}
  \item[(HOM2)'] For all $\bar u\in \UU$ and all $x'\in X'$,
    $f(p'(\bar u,x'))=p(\bar u,f(x'))$.
  \end{description}
  Recall that we are using pointed spaces. The arrows correspond
  to maps which take the selected point to the selected point. The
  selected point in $\UU$ is the empty control signal corresponding
  to~$T=0$. We have not shown the selected point in $Y$ because the
  sensor mappings $h$ and $h'$ do not have a requirement to map
  $x_0$ or $x'_0$ to any particular point, although due to commutation
  of the diagram (if holds), we know that $h'(x_0')=h(x_0)$.
\end{Def}

It is straightforward to see that the existence of a homomorphism
is a sufficient condition for the equivalence to take place:

\begin{Thm}\label{thm:Homomorphism}
  If there is a homomorphism $f\colon E'\to E$, then
  $E\equiv E'$ and $E\equiv^{\II}E'$.
\end{Thm}
\begin{proof}
  Consider the commutative diagram of Definition~\ref{def:Homomorphism}.
  Since the triangle on the left commutes, we have
  $f\circ p'=p$, and from the triangle on the right, we have
  $h'=h\circ f$. So,
  $$h\circ p=h\circ (f\circ p')=(h\circ f)\circ p'=h'\circ p'.$$
  Applying Lemmas~\ref{lemma:TFAE} and~\ref{lemma:TFAE2} we have the
  result.
\end{proof}

\subsection{Covering maps of environments}
\label{sec:Covering}

Covering maps have the role of ``unravelling'' fundamental groups.
A closed loop is the image of a non-closed path in the covering space.
This is why they are natural to analyse loop closure.
The topologist will also readily recognize the idea of covering spaces from the
commutative diagrams in \eqref{eq:Commutative1} and Definition~\ref{def:Homomorphism}. Before we can exploit this idea, we prove the following:

\begin{Lemma}\label{lemma:StrongDeformationRetract}
  The space $\UU$ is contractible.
\end{Lemma}
\begin{proof}
  We will prove a slightly stronger property that $\{\es\}\subset \UU$
  is a strong deformation retract of~$\UU$. Define
  $F\colon \UU\times [0,1]\to \UU$ by
  $F(\bar u,t)=\bar u_{<\theta(t)}$ where $\theta(t)=-\ln(t)$ and by
  convention $-\ln(0)=\infty$ and $\bar u_{<\infty}=\bar u$. By
  Definition~\ref{def:TopoU}, $F$ is continuous because the map
  $(\bar u,t)\mapsto \bar u_{<t}$ is. Clearly, $F(\bar u,0)$ is the
  identity, $F(\es,t)=\es$ for all $t$, and
  $F(\bar u,1)=\bar u_{<0}=\es$ for all $\bar u$. Thus, $F$ is a
  strong deformation retraction to~$\{\es\}$.
\end{proof}

We say that $A\subset X$ is \emph{reachable}, if for all $a\in A$
there is $\bar u\in \UU$ with $p(\bar u)=a$. We say that the
environment is \emph{fully reachable}, if $X$ is reachable.  A
\emph{covering map} $f\colon (X',x'_0)\to (X,x_0)$ is an onto map
which is a local homeomorphism. A pointed space $(X',x')$ is a
\emph{covering space} of $(X,x)$, if there is a \emph{covering map} from
$(X',x')$ to $(X,x)$. 

\begin{Lemma}\label{lemma:Unique}
   Assume that $\UU$ is path-connected and locally path-connected.
  $E=(X,x_0,h,p)$ is a fully reachable environment, and
  $(X',x_0')$ is a covering space of $(X,x_0)$ witnessed by the
  covering map $f$. Then, there is unique initialized path action
  $p'\colon \UU\to X'$ such that $f\circ p'=p$ and a unique 
  sensing mapping
  $h'\colon X'\to Y$ such that $h'=h\circ f$.  
\end{Lemma}
\begin{proof}
  Since $\UU$ is simply
  connected (by Lemma \ref{lemma:StrongDeformationRetract}), there is a
  unique lift of $p\colon(\UU,\es)\to (X,x_0)$ to
  $p'\colon (\UU,\es)\to (X',x'_0)$; see Propositions~1.33 and~1.34 in~\cite{AT}. This $p'$ is a continuous map
  such that $f\circ p'=p$ and $p'(\es)=x'_0$. Thus, it satisfies the conditions of Definition~\ref{def:initialized_path-action}
  and is an initialized path action. The uniqueness and existence of $h'$ follow from its definition.
\end{proof}

\begin{Thm}\label{thm:Lift}
  Assume that $\UU$ is path-connected and locally path-connected, 
  $E=(X,x_0,h,p)$ is a fully reachable environment, and
  $(X',x_0')$ is a covering space of $(X,x_0)$.
  Then, there is a sensorimotor structure $(h',p')$ on $(X',x_0')$ making it into
  environment $E'=(X',x'_0,h',p')$ such that $E\equiv E'$.
\end{Thm}
\begin{proof}
  Let $f$ witness that $(X',x'_0)$ is a covering space of $(X,x_0)$
  and let $h'$ and $p'$ be the functions given by Lemma~\ref{lemma:Unique}.
  Then, they make $f$ into homomorphism $E'\to E$, and the result follows from
  Theorem~\ref{thm:Homomorphism}.
\end{proof}

We assumed above that $\UU$ is path-connected and locally
path-connected. We show that these assumptions are satisfied
for $\UU=\UU_M$ of Definition~\ref{def:PathSpaceBasics}:

\begin{Prop}\label{prop:UUMpclpc}
  Let $\UU_M$ be as in Definition~\ref{def:PathSpaceBasics}
  with metric from Definition~\ref{def:MetricOnU}. 
  Then, $\UU_M$ is path-connected
  and locally path-connected.
\end{Prop}
\begin{proof}
  Let $\bar u_0,\bar u_1\in \UU_M$.  We will construct a path
  $\gamma\colon [0,1]\to \UU_M$ such that $\gamma(k)=u_k$ for
  $k\in \{0,1\}$ and for all $s\in [0,1]$ we will have
  $\varrho_{\UU_M}(\gamma(s),u_0)+\varrho_{\UU_M}(\gamma(s),u_1)=d(u_0,u_1)$.  From
  this it follows that every ball $B_{\UU_M}(\bar u,r)$, $r\in\R_+$, in 
  $\UU_M$ is path connected which implies
  the statement to be proved.  Suppose w.l.o.g. $T_1=|\bar u_1|\ge |\bar u_0|=T_0$.
  For $s\in [0,1]$ let $\gamma(s)$ be a path with $|\gamma(s)|=T_0+(T_1-T_0)s$
  such that for all $t\in [0,T_0+(T_1-T_0)s)$ we have
  $$\gamma(s)(t)=
  \begin{cases}
    \bar u_1(t),\text{ if }t<T_0s\text{ or }t>T_0\\
    \bar u_0(t),\text{ otherwise}.
  \end{cases}
  $$
  Being piecewise measurable, $\gamma$ is measurable; thus, $\gamma\in\UU_M$.
  Using the fact that $0\le s\le 1$ one can verify that
  $|T_0+(T_1-T_0)s-T_0|+|T_1-(T_0+(T_1-T_0)s)|=|T_1-T_0|$. 
  Then
  \begin{align*}
    &\varrho_{\UU_M}(\gamma(s),u_0)+\varrho_{\UU_M}(\gamma(s),u_1) \\
    &=\int_0^{T_0}d_U(\gamma(s)(t),u_0(t))dt+\int_0^{T_0}d_U(\gamma(s)(t),u_1(t))dt + |T_1-T_0|\\
    &=\int_0^{sT_0}d_U(\gamma(s)(t),u_0(t))dt +\int_{sT_0}^{T_0}d_U(\gamma(s)(t),u_0(t))dt\\
    &\phantom{===}+\int_0^{sT_0}d_U(\gamma(s)(t),u_1(t))dt +\int_{sT_0}^{T_0}d_U(\gamma(s)(t),u_1(t))dt + |T_1-T_0|
  \end{align*}
  %splitting environments to allow for pagebreak
  \begin{align*}
    &=\int_0^{sT_0}d_U(u_1(t),u_0(t))dt +\int_{sT_0}^{T_0}d_U(u_0(t),u_0(t))dt \\
    &\phantom{===}+\int_0^{sT_0}d_U(u_1(t),u_1(t))dt +\int_{sT_0}^{T_0}d_U(u_0(t),u_1(t))dt + |T_1-T_0|\\
    &=\int_0^{sT_0}d_U(u_1(t),u_0(t))dt +\int_{sT_0}^{T_0}d_U(u_0(t),u_1(t))dt + |T_1-T_0|\\
    &=\int_0^{T_0}d_U(u_1(t),u_0(t))dt  + |T_1-T_0|\\
    &=\varrho_{\UU_M}(u_0,u_1),
  \end{align*}
  which was to be proven.  
\end{proof}

\begin{Cor}\label{cor:Lift}
  Assume that $\UU=\UU_M$, $E=(X,x_0,h,p)$ an environment, and $(X',x_0')$ a covering
  space of $(X,x_0)$. 
  Then, there is a sensorimotor structure $(h',p')$ on $(X',x_0')$ making it into
  environment $E'=(X',x'_0,h',p')$ such that $E\equiv E'$.
\end{Cor}
\begin{proof}
    By Lemma~\ref{prop:UUMpclpc}, $\UU_M$ satisfies the assumptions of Theorem~\ref{thm:Lift}.
\end{proof}

\begin{Open}\label{Q:Conditions}
  What are the minimal topological conditions for $\UU$ such that
  Theorem~\ref{thm:Lift} holds and when are they satisfied?
\end{Open}

\begin{Thm}\label{thm:CoveringEquiv}
  Suppose $E=(X,x_0,h,p)$ and $E'=(X',x_0',h',p')$ are environments
  and that one of the following holds:
  \begin{enumerate}
  \item[(A)] There exists a common covering space $(\tilde X,\tilde x)$ of
    both $(X,x_0)$ and $(X',x'_0)$ witnessed by covering maps $f$ and
    $f'$ respectively such that the lifts of $p$ and $p'$ are
    identical and such that $h\circ f=h'\circ f$.
  \item[(B)] There exists a space $(\hat X,\hat x_0)$ such that both
    $(X,x_0)$ and $(X',x'_0)$ are covering spaces of
    $(\hat X,\hat x_0)$ witnessed by covering maps $f$ and $f'$
    respectively such that $f\circ p=f'\circ p'$ and there is
    $\hat h\colon \hat X\to Y$ such that $h$ and $h'$ are lifts of
    $\hat h$ along the respective covering maps.
  \end{enumerate}
  Then, $E\equiv E'$.
\end{Thm}
\begin{proof}
  In the first case, denote by $\tilde p$ the lift of $p$ (and $p'$) and by
  $\tilde h=h\circ f=h'\circ f$. Then, it is clear that $f$ and $f'$
  are homomorphisms from
  $\tilde E=(\tilde X,\tilde x_0,\tilde h,\tilde p)$ to $E$ and $E'$
  respectively. Thus, we have $\tilde E\equiv E$ and $\tilde E\equiv E'$.
  By transitivity  we have $E\equiv E'$. Note that the fact that
  $\equiv$ is an equivalence relation follows from
  Corollary~\ref{cor:ER} and Theorem~\ref{thm:EquivEquiv}. 

  For the second case, denote $\hat p=f\circ p=f\circ p'$ and we have
  that $f$ and $f'$ are homomorphisms from $E$ and $E'$ respectively
  to $\hat E=(\hat X,\hat x_0,\hat h,\hat p)$; then apply 
  transitivity again.
\end{proof}

\begin{Example}[An application to topology]\label{ex:ApplicationToTop}
  We can use Corollary~\ref{thm:CoveringEquiv} to prove that certain
  spaces \emph{do not} have a common covering space, if we can show
  that a robot can distinguish between them. For example consider a
  robot which can detect the local homeomorphism type of its
  environment (for more on this see Theorem~\ref{thm:IsometryInvariant}
  in Section~\ref{sec:Revisiting}). Then, this robot can easily detect a
  difference between the following two spaces on Figure~\ref{fig:AppToT}.
  In the space on the left it is possible to go forward along a
  1-dimensional path and repeatedly bump into a 4-crossing 
  (by circling around the left-most loop). In the space on the right,
  however, no matter how the robot traverses along the 1-dimensional edges,
  it will sooner or later
  bump into a 3-crossing. U-turns midway are not made. Thus, $E\not\equiv E'$ and therefore by
  Corollary~\ref{thm:CoveringEquiv} there is no common covering space
  of both of them, neither there is a space which both of them are a
  covering space of.
\end{Example}

\subsection{Strong indistinguishability}
\label{ssec:Strongly}

Motivated by Theorem~\ref{thm:Lift}, we can define:

\begin{Def}
  A pointed space $(X,x_0)$ is \emph{strongly indistinguishable} from
  the pointed space $(X',x'_0)$, if for all sensorimotor structures
  $(h,p)$ on $(X,x_0)$ there is a sensorimotor structure $(h',p')$ on
  $(X',x_0')$ such that $E\equiv E'$ where $E=(X,x_0,h,p)$ and
  $E'=(X',x'_0,h',p')$.
\end{Def}

The relation of strong indistinguishability is not symmetric, but it is transitive and reflexive; thus, it determines a quasiorder on environments. The higher an environment is in this ordering 
the more ambiguities it has. At the top are the universal covering spaces and at the bottom
are the quotients with respect the equivalence relation $h(x)=h(x')$ (see~\cite{WeiSakLav22} for related concepts).
See Figure~\ref{fig:Indistinguishable} and it's caption.

\section{Equivalence Characterization and Bisimulation}
\label{sec:Bisim}

Theorem~\ref{thm:CoveringEquiv} gives sufficient conditions to decide
when two spaces are $\equiv$-equivalent via covering spaces and covering maps.
It is also appealing from the point of robotics because a visual sensor
will typically report information about the geometric structure of the environment
and covering spaces have the property of preserving the local topological structure.
The covering maps, however, do note characterize
the equivalence. For example, it is not hard to come up with examples
where a projection map is in the role of a homomorphism, or where no homomorphism
exists.

\begin{Example}
  For example let $X=S^1$ and $X'=S^1\times [0,1]$.
  Let $h\colon X\to \R$ and $h'\colon X'\to R$ be defined so that
  for all $\theta\in [0,2\pi)$ and all $t\in [0,1]$,
  $$h(e^{i\theta})=h'(e^{i\theta},t)=\sin(\theta).$$
  Thus, $h'$ does not depend on $t$. Further, let the initialized
  path actions be some functions $p\colon\UU\to X$ and
  $p'\colon\UU\to X'$ such that $\pr_1\circ p'=p$ where $\pr_1$ is the
  projection to the first coordinate. Now the projection mapping $\pr_1$
  is, in fact, a witness of the equivalence between these two
  environments, but of course there are no common covering spaces because these must be local homeomorphisms.  
\end{Example}

It is even possible to have two spaces which are equivalent but
there is no map at all that witnesses that:

\begin{Example}
  Let $X=X'=\Delta\lor I$ where $\Delta$ is a 2-simplex and $I$ is a
  1-simplex, and $\lor$ means that they are glued at one point, for
  example at a $0$-face of both. Thus, it looks like a kite.  Suppose
  $h\colon X\to \{-1,0,1\}$ is such that $h(x)=-1$ for all $x$ in the
  2-simplex and $h(x)=1$ for $x$ in the 1-simplex, except $h(x)=0$ at
  the point where they are glued. For $h'$ the numbers are flipped so
  that $h'=-h$. Now if a continuous $f\colon X\to X'$ commutes with
  the sensor mappings, then it is not surjective. It is not hard to
  come up with initialized path actions, however, which make the
  environments both equivalent and fully reachable. See Example~\ref{ex:Beams2}
  for another pair of such environments.
\end{Example}

As a solution we will use a notion of bisimulation in the continuous
setting. There is a lot of literature on bisimulation in the
topological and continuous setting, especially in the context of
hybrid systems~\cite{Cuijpers2004, Davoren, ferns2011bisimulation,
  henzinger1995s}.

\begin{Def}
  Let $E=(X,x_0,h,p)$ and $E'=(X',x'_0,h',p')$ be environments
  where $p$ and $p'$ are path actions (not initialized ones, see
  Definition~\ref{def:Environment}).  A binary relation
  $R\subset X\times X'$ is a \emph{bisimulation} between $E$ and $E'$,
  if $(x_0,x'_0)\in R$ and for all $(x,x')\in X\times X'$ the
  following holds: If $(x,x')\in R$, then $h(x)=h(x')$ and for all
  $\bar u\in\UU$, also $(p(\bar u,x),p(\bar u,x'))\in R$.
\end{Def}

\begin{Thm}\label{thm:Bisimulation}
  Suppose $E=(X,x_0,h,p)$ and $E'=(X',x'_0,h',p')$ are environments
  where $p$ and $p'$ are path actions (not initialized ones, see
  Definition~\ref{def:Environment}) and assume that $\UU$ is closed
  under~$\oplus$. Then, $E\equiv E'$ if and only if there is a bisimulation
  $R\subset X\times X'$.  
\end{Thm}
\begin{proof}
  We work with $\equiv^{\II}$ instead of $\equiv$ as justified
  Suppose $E\equiv^{\II} E'$. Then, define
  $R=\{(p(\bar u,x_0),p'(\bar u,x'_0))\mid \bar u\in\UU\}$.
  By choosing $\bar u=\es$, we have $(x_0,x'_0)\in R$.
  Suppose $(x,x')\in R$. Let $\bar u\in\UU$ witness this.
  Then
  $$h(x)=h(p_{x_0}(\bar u))=h'(p'_{x_0}(\bar u))=h'(x').$$
  Here we used the notation from Remark~\ref{rem:FromPathTOInitPath}
  and Lemma~\ref{lemma:TFAE}(4.). Now let $\bar u_1$ be arbitrary.
  Then, by Definition~\ref{def:path-action}(PA2) and Remark~\ref{rem:Basics}(4)
  we have
  $$p(\bar u_1,x)=p(\bar u_1,p(\bar u,x_0))=p(\bar u\oplus \bar u_1,x_0)$$
  and
  $$p'(\bar u_1,x)=p(\bar u_1,p'(\bar u,x_0))=p'(\bar u\oplus \bar u_1,x_0).$$
  Thus, $\bar u\oplus\bar u_1$ witnesses that $(p(\bar u_1,x),p(\bar u_1,x'))\in R$.
  This completes the proof of the ``only if''-part.

  Suppose now that $R$ is a bisimulation relation on $X\times X'$.
  Then, by definition of bisimulation $h(x_0)=h(x'_0)$, and for all
  $\bar u$, also $h(p(\bar u,x_0))=h'(p'(\bar u,x_0'))$. By
  Lemma~\ref{lemma:TFAE}, we are done.
\end{proof}

\begin{Open}
 Bisimulation often arises in the context of modal
 logic and Kripke models~\cite{van2010modal} and has also
 been studied by the present authors
 in the context of robotics and minimal sufficient equivalence relations~\cite{WeiSakLav22}. Can a closer connection to these
 areas established by Theorem~\ref{thm:Bisimulation}
 be made?
\end{Open}

%\section{Circling Back}
\section{Bringing It All Together}
\label{sec:Revisiting}

%Metaphorically speaking, w
We now circle back to the main questions of interest, raised in Sections \ref{sec:intro} and~\ref{sec:Loops}. An important special case of a sensor mapping $h$ 
is one which reports invariant information about the topological or metric properties of
the local neighbourhood of the agent. Call such sensor mapping \emph{geometry based}.
This is a typical function of distance measurements
and visual sensors in general. Covering maps preserve local topological structure,
and if additionally required to be local isometries, also local metric structure. 
Therefore covering maps naturally preserve geometry based sensor mappings.
This enables applying our framework to a diverse number of cases in theoretical 
robotics and we synthesize it in Theorem~\ref{thm:IsometryInvariant} 
and Corollary~\ref{cor:IsometryInvariant} below. These results can be seen as a
culmination of this paper and the original motivation to explore this topic.

Recall the setup of Section~\ref{ssec:SpaceOfTrajectories}. Let $O(X)$ be the set of open
subsets of a Polish space $X$ and suppose that $\xi\colon X\to O(X)$ is some
\emph{neighborhood function}, meaning that for all $x\in X$, we have
$x\in \xi(x)$. We will consider systems $(X,x_0,h,p)$, where $h(x)$ is
either a metric or a topological invariant of $\xi(x)$, meaning
that $\xi(x)\sim \xi(x')\rightarrow h(x)=h(x')$, where $\sim$
is either isometry or homeomorphism.  The idea is
that $\xi(x)$ is a set visible from $x$, and $h$ is some sensor mapping
that loses unnecessary information. For example, we revisit the
gap-navigation trees of \cite{TovMurLav07} (recall
Section~\ref{ssec:MinimalFiltering}).

\begin{Example}\label{ex:GapNavTrees}
  In the gap-navigation trees setup, $X$ is a closed subset of $\R^2$,
  and
  $$\xi(x)=\{x'\in X\mid [x,x']\subset X\}$$
  is the set of all points reachable by a line from the robot's
  position as
  depicted on Figure~\ref{fig:StarConvexInvariant}(left). This set
  is open in $X$, but not open in~$\R^2$. In fact, we can see that
  $\xi(x)$ is homeomorphic to a set $A$ that has the property
  $$B^2(0,1)\subseteq A\subseteq \bar B^2(0,1),$$
  meaning that it is the closed 2-disk with some parts of the boundary
  missing. The missing parts of the boundary correspond precisely to
  the gaps in the visual field, or the discontinuities in the distance
  function. Thus, the sensor mapping $h(x)=g(\xi(x))$, where $g$ reports the circular
  order of these discontinuities is a topological invariant
  of~$\xi(x)$. The paper \cite{TovMurLav07} addresses environments that are
  not simply connected, 
  which the authors handle by
  having the robot place distinguishable pebbles in the environment at selected locations.
  Now we can go a step further and utilize
  Theorem~\ref{thm:CoveringEquiv} to construct indistinguishable
  but non-homeomorphic environments for such a sensor, one can start
  with a region which is not simply connected, take its covering space such that
  the covering map preserves the star-convex neighbourhoods up to
  homeomorphism.  Too much distortion and they can be distinguished.
  Consider Figure~\ref{fig:GNTCov}. We denote the spaces depicted in (a), (b),
  and (c) respectively by $A$, $B$ and~$C$.
  The environment $B$ is a covering space
  of $A$ with a covering map which preserves the homeomorphism type of the star convex
  neighborhoods, i.e., commutes with the $\xi$-mapping defined above. 
  $C$ is also a covering space of $A$, but the covering map does not
  commute with $h=g\circ \xi$ and so is distinguishable from~$A$,
  and therefore also from~$B$.
  This is witnessed by the star convex neighborhood of the point in the top corner in
  Figure~\ref{fig:GNTCov}(d). This neighborhood has four discontinuities and such is nowhere
  to be found in $A$, $B$ or $C$. According to Theorem~\ref{thm:Lift}
  there exists some other sensor mapping $h'$ on $C$ which makes it indistinguishable
  from $A$. %Thus, there exist ``VR-glasses'' which, if the robot wears them in $C$, will make it indistinguishable from~$A$.
\end{Example}

\begin{figure}
    \centering
    \includegraphics{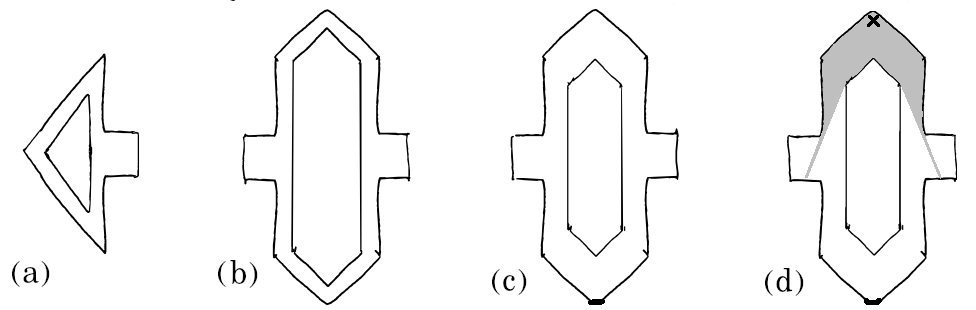}
    \caption{Environments (a) and (b) are indistinguishable by the gap-navigation sensor,
    but environments (a) and (c) as well as (b) and (c) are distinguishable. This is witnessed by the
    star convex neighborhood depicted in (d).}
    \label{fig:GNTCov}
\end{figure}

\begin{Example}\label{ex:WallFollowing}
  In the wall-following robot scenario of~\cite{Katsev2011}
  (Figure~\ref{fig:ConvexReflex}), the robot can sense only
  very locally. The polygonal environments in \cite{Katsev2011} are \emph{metrically locally uniform} in the following sense:
  \begin{description}
  \item[Local metric uniformity.] $X$ is equipped with a metric $d_X$ and for
    all $x\in X$ there is $\e_x>0$ such that for all $0<\delta<\e_x$ the
    subspaces $\bar B_X(x,\e_x)$ and $\bar B_X(x,\delta)$ are isometric. We call the
    isometry type of $\bar B_X(x,\e_x)$ the \emph{local isometry type at~$x$}.
  \end{description}
  Thus, let $\xi(x)=B_X(x,\e_x)$. Then, $\xi$ encodes local metric
  structure. The sensor of~\cite{Katsev2011} is now a metric 
  invariant of~$\xi$. Thus, in view of 
  Theorem~\ref{thm:CoveringEquiv}, if one
  wanted to construct environments indistinguishable 
  by such sensor, one could start with covering spaces whose 
  covering maps are local isometries.
\end{Example}

\begin{Example}\label{ex:GraphExplo}
  The problem of detecting graph isomorphism by exploring it
  \cite{Kuipers1988ARQ,Thrun1998} is a problem of reconstructing
  a global map from local information. Graphs, viewed as 1-complexes,
  are \emph{topologically locally uniform} in the following sense:
  \begin{description}
  \item[Local topological uniformity.] $X$ has a basis $B$ such that
    for all $x\in X$ there is $O_x\in B$ such that for all
    $O_1\subset O_x$ with $O_1\in B$, $O_1$ is homeomorphic to~$O_x$.
    We call the homeomorphism type of $O_x$ the \emph{local homeomorphism type of~$x$}.
  \end{description}
  Thus, now letting $\xi(x)=O_x$, it encodes the local homeomorphism type
  around the point. In graphs this will be either a straight line or a
  node of some degree $d\in\N$. Metric realizations of graphs can have
  edges of varying length which should be ignored. One option is to
  use the topological history information equivalence (Example~\ref{ex:Toppo}).
  The other option is to equip the complex first with a metric $d$ in
  which all the nodes of degree $d$ have isometric neighbourhoods and
  so that the distance between nodes equals one, and then redefine
  $d'(x,y)=\min\{d(x,y),\frac{1}{2}\}$ to lose all non-local
  information.

  By Theorem~\ref{thm:CoveringEquiv}, if the local homeomorphism type is all the robot can ever
  see, it cannot distinguish between 1-complexes which have the same
  universal covering. However, it \emph{can} distinguish between those that
  do not have. Most algorithms in this area will exploit other tools
  such as edge-labeling or pebble placing, for the robot to recognize
  which nodes have been visited already.
\end{Example}

\begin{Open}
    Develop algorithms for a robot to distinguish between
    non-$\equiv$-equivalent 1-complexes without edge labeling
    or pebble placing.
\end{Open}

  \begin{Open}\label{Q:Pebble}
    Can our present theory 
    elegantly accommodate the ``pebble placing''? It seems that even a
    single pebble can significantly narrow down the space of possible
    worlds in which the robot could find itself.
  \end{Open}

\begin{Example}\label{ex:Beams2}
  Consider the example of beam sensing~\cite{Tovar2009} of
  Figure~\ref{fig:Beams}. A robot moves in a multiply connected environment
  and whenever it crosses a beam, it senses the label of this beam. In 
  \cite{Tovar2009} the authors show that under certain assumptions this helps
  the robot to determine homotopy invariants of its own trajectory. One of the 
  assumptions is that each beam has a unique label. By dropping this assumption
  we can, using Corollary~\ref{cor:Lift}, design various environments which will be
  indistinguishable from the perspective of such a robot. We show some examples
  in Figure~\ref{fig:Beam2}. To see that the spaces (b) and (c) are covering spaces
  of (a), we refer to~\cite[p.~58]{AT}. Note that by Corollary~\ref{cor:Lift},
  \emph{any} arrangement of beams in (a) can be lifted to an arrangement of beams
  in (b) and (c) so that the environments become indistinguishable. Whereas
  (b) and (c) are also $\equiv$-equivalent (Theorem~\ref{thm:CoveringEquiv}(B)),
  neither one is a covering of the other one; thus, not all arrangements of beams
  on (b) can be lifted to (c), or vice versa. Environments (b) and (c) are examples
  of equivalent ones between which there is no map witnessing the equivalence.
  Only a many-to-many bisimulation witnesses the equivalence (Theorem~\ref{thm:Bisimulation}).
\end{Example}

\begin{figure}
  \centering
  \includegraphics[width=0.9\textwidth]{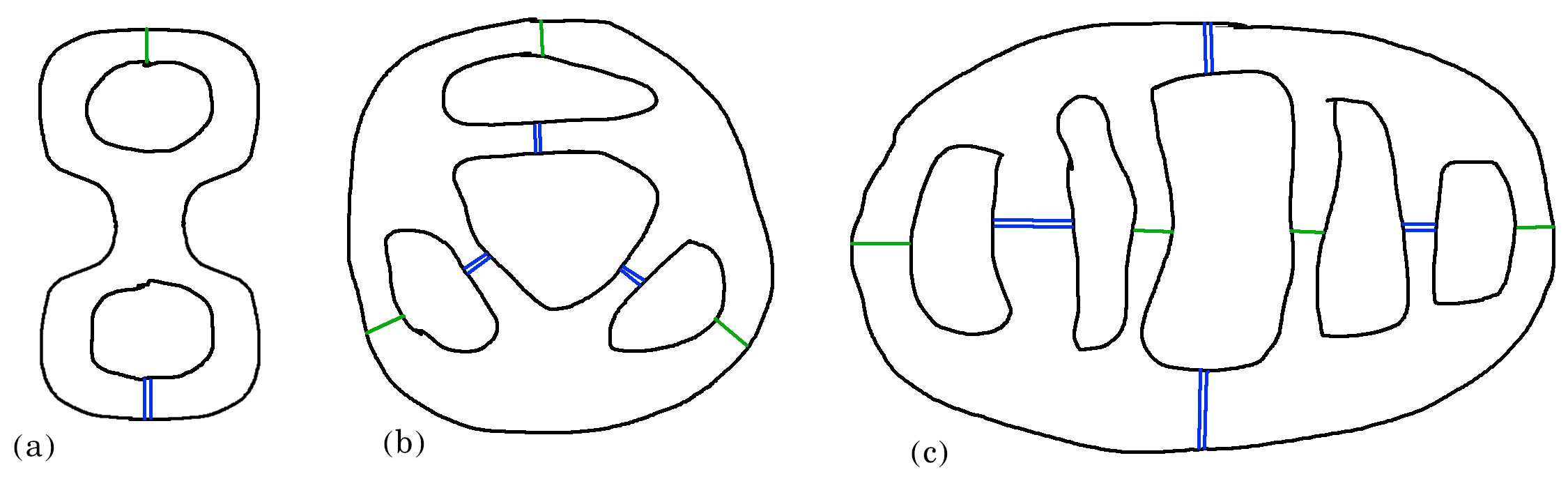}
  \caption{Indistinguishable environments for the beam-detecting robot. In each 
  there are two types of beams, a single (green) and a double (blue) beam.
  All three environments are $\equiv$-equivalent, and (b) and (c) are covering spaces
  of (a). Therefore, any arrangement of beams on (a) can be ``lifted'' to (b) and
  (c), making them indistinguishable.}
  \label{fig:Beam2}
\end{figure}

Motivated by the idea of local uniformity expressed in Examples 
\ref{ex:WallFollowing} and \ref{ex:GraphExplo}, we can formulate a general theorem. We will formulate
it for the metric uniformity, leaving topological uniformity for future work.
Let $P$ be the space of all Polish metric spaces (for example viewed as 
the Effros space of all closed subsets of the Urysohn space \cite{Gao2008,Kechris1995}).
For any 
Polish metric space $X\in P$ which is locally metrically uniform,
let $\xi^X\colon X\to P$ be a function such that 
$\xi^X(x)=\bar B_X(x,\e_x)$ where $\e_x$ is a witness for local metric
uniformity at $x$. Since the set of such $\e_x$ is a connected subset of $\R_+$
(since it is downward closed), it is $K_\sigma$ and so $\xi^X$ can always
be chosen to be Borel by the Arsenin-Kunugui uniformization 
theorem~\cite[18.18]{Kechris1995}. Note that since $\bar B_X(x,\e_x)$
is a closed subset of $X$, it is also a member of~$P$.
By a metric covering map $f$ we mean a covering
map which is a local isometry.
Fix $g\colon P\to Y$ to be any
isometry-invariant function which means that $g(M)\ne g(M')$ implies that $M$ and $M'$
are not isometric.  Recall that $\UU_M$ is the space of all measurable control signals
(Definition~\ref{def:PathSpaceBasics}). Below all path actions are assumed to have
the domain~$\UU_M$.

\begin{Thm}\label{thm:IsometryInvariant}
    Suppose $X$ and $X'$ are locally metrically uniform Polish spaces, $h=g\circ \xi^{X}$ and $h'=g\circ \xi^{X'}$,
    and that $p$ and $p'$ are initialized path actions
    on $X$ and $X'$ with $p(\es)=x_0$ and $p'(\es)=x_0'$. Suppose there is a metric covering map 
    $f\colon (X',x_0')\to (X,x_0)$ such that $f\circ p'=p$.
    Then, the environments $(X,x_0,h,p)$ and $(X',x_0',h',p')$ are $\equiv$-equivalent.
\end{Thm}
\begin{proof}
    We need to show that $h'=h\circ f$. To see this, let $x'\in X'$. Let $\e$ be small enough such that
    $f\rest B(x',\e)$ is an isometry, and for all $\delta<\e$, $B_X(x',\d)$ is isometric
    to $\xi^{X'}(x')$ and $B_{X'}(f(x),\delta)$ is isometric to $\xi^X(f(x'))$.
    However, an isometry takes balls to balls of the same radius; thus, $\xi^X(f(x'))$ is isometric
    to $\xi^{X'}(x')$ which implies $(g\circ \xi^{X'})(x')=(g\circ \xi^X)(f(x))$
    by the property that $g$ is invariant with respect to isometry.
    By the definition of $h$ and $h'$ this means $h'(x')=h(f(x))$.
\end{proof}

\begin{Cor}\label{cor:IsometryInvariant}
    Suppose $X$ and $X'$ are locally metrically uniform Polish spaces, $h=g\circ \xi^{X}$ and $h'=g\circ \xi^{X'}$,
    and that $p$ is an initialized path action
    on $X$ with $p(\es)=x_0$. Suppose there is a metric covering map 
    $f\colon (X',x_0')\to (X,x_0)$ for some $x_0'\in X'$. 
    Then there is a path action $p'$ on $X'$ such that 
    the environments $(X,x_0,h,p)$ and $(X',x_0',h',p')$ are $\equiv$-equivalent.
\end{Cor}
\begin{proof}
    Using Corollary~\ref{cor:Lift}, let $p'$ be a lifting of $p$ and apply
    Theorem~\ref{thm:IsometryInvariant}.
\end{proof}

\begin{Open}
   Prove results for topological uniformity that are analogous to Theorem~\ref{thm:IsometryInvariant}
   and Corollary~\ref{cor:IsometryInvariant}.
\end{Open}

\section{Conclusion}

This paper has formalized and unified previous notions pertaining to robots that explore unknown environments using limited sensors.
We started by motivating a mathematical analysis of the robotics loop closure problem. Perhaps our approach best describes the situation of
false positives in loop closure. Indeed, if $p(\bar u)\ne p'(\bar u)$,
but $h(p(\bar u))=h'(p'(\bar u))$, then we might be tempted to infer that
loop closure was detected, but it is a false alarm. This false
positive at the extreme is tantamount to the inability to distinguish
between the environments completely, that is when $h\circ p=h'\circ p'$.
Our intuition was that this is closely related to the idea of covering
spaces because covering maps literally \emph{close loops} by mapping,
at best, contractible spaces onto spaces with non-trivial fundamental groups.
With this motivation, we developed a general topological theory
that relates control signals, trajectories, and path actions in Section~\ref{sec:GeneralTheory}.
Building on the framework in \cite{Yershov2010} and the general theory of dynamical systems, we defined a continuous-time version of
history information
spaces. 
Then, we applied the resulting tools
to define various equivalence relations on environments, which are of independent topological and
set theoretic interest (recall
Open Problem~\ref{Q:Complexity}). We then moved on to the main motivation, covering spaces, and proved that if $X'$ is a covering space of $X$, then indeed
any sensorimotor structure can be lifted from $X$ to $X'$,
making it look exactly the same from the point of view
of a robot (Theorem~\ref{thm:Lift}). For this theorem
we assumed that $\UU$ is path-connected and locally path-connected.
We proved this for the space of measurable
controls $\UU_M$ (Proposition~\ref{prop:UUMpclpc}),  but we left 
open whether these
conditions can be weakened, and by how much (Open Problem~\ref{Q:Conditions}).
Covering spaces and
covering maps may be attractive from the point of view of
applications as they preserve local structure and are convenient
to work with, but they
do not give a complete mathematical characterization of the indistinguishability
relation. This is why in Section~\ref{sec:Bisim} we used the
notion of bisimulation to obtain a necessary and sufficient condition for the 
equivalence to take place. The final result was Theorem \ref{thm:IsometryInvariant}
along with Corollary~\ref{cor:IsometryInvariant}, which unifies and explains several robot navigation settings, by characterizing environment ambiguities in terms of covering spaces.
%Finally, Section~\ref{sec:Revisiting}
%revisited the minimalist examples of Section~\ref{sec:Loops}
%with our new tool box. 

\small
\bibliographystyle{plain}
\bibliography{Robo.bib}

\end{document}